\documentclass[times,11pt]{article}
\usepackage{fullpage}
\usepackage{natbib}
\usepackage{hyperref}
\hypersetup{colorlinks = true, citecolor = blue} 
\newcommand{\nop}[1]{}
\usepackage{pifont}
\usepackage{mathrsfs}
\usepackage{epsfig}
\usepackage{graphics}
\usepackage{epsf}
\usepackage{amsmath, amsthm, amssymb, multirow, paralist, bm}
\usepackage{fullpage,color}
\usepackage{url}
\usepackage{algorithm,algorithmic}
\newtheorem{thm}{Theorem}

\newtheorem{lemma}{Lemma}
\newtheorem{cor}[thm]{Corollary}

\newtheorem{assumption}{Assumption}
\newtheorem{remark}{Remark}
\PassOptionsToPackage{numbers, compress}{natbib}

\def \R {\mathbb{R}}

\def \N {\mathcal{N}}
\def \A {\mathcal{A}}

\def \I {\mathcal{I}}

\newcommand{\nm}[1]{\left\lVert#1\right\rVert}

\newcommand{\paren}[1]{\left({#1}\right)}
\newcommand{\brackets}[1]{\left[{#1}\right]}
\newcommand{\braces}[1]{\left\{{#1}\right\}}

\DeclareMathOperator{\Tr}{Tr}

\usepackage{hyperref}
\usepackage{bbm}
\usepackage{adjustbox}
\usepackage{makecell}
\usepackage{float}
\usepackage[normalem]{ulem}
\restylefloat{table}
\usepackage[x11names,dvipsnames,table]{xcolor}
\usepackage{mathtools}
\usepackage{booktabs} 
\usepackage{tabularx}
\usepackage{caption}
\usepackage{subcaption}
\usepackage{wrapfig}
\newcolumntype{P}[1]{>{\centering\arraybackslash}p{#1}}
\begin{document}
\title{
A Convergence Theory Towards Practical \\Over-parameterized Deep Neural Networks}

\author{Asaf Noy$^*$, Yi Xu$^*$, Yonathan Aflalo$^*$, Lihi Zelnik-Manor$^*$, Rong Jin\thanks{Equal contribution.}\\ 
Machine Intelligence Technology, Alibaba Group\\
\{asaf.noy, yixu, jonathan.aflalo, lihi.zelnik, jinrong.jr\}@alibaba-inc.com
}
\date{}
\maketitle
\begin{abstract} \noindent
Deep neural networks' remarkable ability to correctly fit training data when optimized by gradient-based algorithms is yet to be fully understood.
Recent theoretical results explain the convergence for ReLU networks that are wider than those used in practice by orders of magnitude. 
In this work, we take a step towards closing the gap between theory and practice by significantly improving the known theoretical bounds on both the network width and the convergence time. 
We show that convergence to a global minimum is guaranteed for networks with widths quadratic in the sample size and linear in their depth at a time logarithmic in both. 
Our analysis and convergence bounds are derived via the construction of a surrogate network with fixed activation patterns that can be transformed at any time to an equivalent ReLU network of a reasonable size. 
This construction can be viewed as a novel technique to accelerate training, while its tight finite-width equivalence to Neural Tangent Kernel (NTK) suggests it can be utilized to study generalization as well.
\end{abstract}
\section{Introduction}
Deep neural networks have achieved remarkable success in machine learning applications of different fields such as computer vision~\citep{voulodimos2018deep}, speech recognition~\citep{hinton2012deep}, and natural language processing~\citep{devlin2018bert}. 
Much of this success is yet to be fully explained.
One of the existing gaps is the network's \textit{trainability}, its ability to perfectly fit training data when initialized randomly and trained by first-order methods.
Modern deep networks are commonly equipped with rectified linear unit (ReLU) activations~\citep{xu2015empirical}, forming highly non-convex and non-smooth optimization problems. Such optimization problems are known to be generally computationally infeasible~\citep{livni2014computational}, and NP-hard in some cases~\citep{blum1989training}. Nevertheless, in practice, trainability is often achieved in various tasks. Such empirical findings were summarized by
~\cite{zhang2016understanding}:
``Deep neural networks easily fit (random) labels".

Previous theoretical works focused on networks' \textit{expressivity}~\citep{bengio2011expressive,cohen2016expressive}, the existence of a solution that perfectly fits training data, therefore necessary for trainability. As it has been shown that even shallow nonlinear networks are universal approximators~\citep{hornik1989multilayer}, a better understanding of the power of depth was desired~\citep{eldan2016power,liang2016deep,telgarsky2015representation}.
One conclusion is that increased depth allows exponentially more efficient representations, thus per parameter, deep networks can approximate a richer class of functions than shallow ones~\citep{cohen2016expressive,liang2016deep}.
However, a recent work by~\cite{yun2019small} provided tighter expressivity bounds on ReLU networks, showing that datasets of $n$ training examples can be perfectly expressed by a $3$-layer ReLU network with $\Omega{\paren{\sqrt{n}}}$ parameters. This does not explain
the sizes of practical neural networks that are typically \textit{over-parameterized} and exceed the training data's size. 
It appears that expressivity can provide only a limited explanation for the overparameterization of modern neural networks, which keep growing deeper and wider in search of state-of-the-art results~\citep{zagoruyko2016wide,huang2019gpipe,ridnik2020tresnet}. 

Important insights regarding trainability emerged from the analysis of simplified variants of neural networks.
Deep \emph{linear} networks are of special interest since increased depth does not affect their expressiveness, only changes their optimization landscape. Therefore the effects of increased width and depth on the training process can be isolated and carefully studied in this setting. 
\cite{arora2018convergence} proved that trainability is attained at a linear rate under proper initialization as long as the network is wide enough. 
\cite{arora2018optimization} showed that training with gradient-descent could be accelerated by increasing the network's depth. Their empirical evaluations supported the existence of a similar outcome in deep nonlinear networks as well.
Another simplified variant with further insights is
overparameterized \emph{shallow} nonlinear networks, typically with a single hidden layer.
\cite{du2018power} proved that for a quadratic activation function all local minima are global minima.
\cite{safran2018spurious}
showed that ReLU networks suffer from spurious local minima, which can be drastically reduced by overparameterization.
\cite{oymak2020towards} achieved trainability where the proportion between the number of hidden units and training examples $n$ depends on the activation: for smooth ones it is $\Omega\paren{n^2}$ while for ReLU it is significantly larger, $\Omega\paren{n^4}$. Additional works examined the unique dynamics of ReLU activation during the training and their relation to gradient-descent optimization~\citep{li2018learning,arora2019fine}.

The analysis of \emph{deep} ReLU networks is more challenging, as additional problems emerge with increased depth. Exploding and vanishing gradient~\citep{hanin2018neural} and bias shifts~\citep{clevert2015fast} are shared with additional activations, while the dying ReLUs problem, causing them to remain inactive permanently, is unique~\citep{lu2019dying}.
While in practice those problems are solved by the introduction of additional tensor transformations~\citep{ioffe2015batch, he2016identity}, their theoretical understanding remains limited. 
Correspondingly, existing trainability guarantees for deep ReLU networks are considerably weaker compared with shallow networks.
\cite{du2019gradient} required a minimal network width which scales exponentially with its depth $L$ and polynomially  with the number of examples $n$. More recent works~\citep{zou2018stochastic,allen2019convergence} improved the dependencies to be high-order polynomials, $\Omega\paren{n^{26}L^{38}}$ and $\Omega\paren{n^{24}L^{12}}$ correspondingly. 
The best known result, by 
\cite{zou2019improved}, required a network width of $\Omega\paren{n^{8}L^{12}}$, 
which is still prohibitively large compared to practice. 
For instance, training their network with $1001$ layers~\citep{he2016identity} over the common ImageNet dataset~\citep{krizhevsky2017imagenet}
requires a network of $\Omega\paren{10^{172}}$ parameters, while the largest one reported contains only $O\paren{10^{11}}$~\citep{brown2020language}.

In this work, we take a step towards closing the gap between theory and practice. We develop a novel technique to analyze the trainability of overparameterized deep neural networks and provide convergence guarantees under size requirements that are significantly milder than previously known.  
The network width required by our analysis for trainability scales only \textit{quadratically} with $n$, enabling for the first time empirical evaluation of the theoretical bounds, and bridging previous empirical results related to overparameterization. 
In addition, the required width by our theory is \textit{linear} with $L$, paving the way for a better understanding of the behavior of deep neural networks of practical depths. 
Finally, the number of training iterations for reaching global minima by our theory is \textit{logarithmic} in $nL$, significantly smaller than previous theoretical results
and similar to leading practical applications (e.g., BiT~\citep{kolesnikov2019big}).
A full comparison with previous methods can be found in Table~\ref{table:theo_comparison}.

A key novelty of our analysis is the construction of a surrogate network, named \textit{Gated ReLU} or \textit{GReLU}, illustrated in Figure~\ref{fig:network_blocks}.
The activation patterns of GReLU, i.e. which entries output zero~\citep{hanin2019deep}, are determined at initialization and kept fixed during training. They are set by a random transformation of the input, independent of the network
initialization. 
Therefore, GReLU networks are immune to two main problems of training ReLU networks, dying ReLUs and bias shifts, while still enjoying the advantages of ReLU activation.
We first prove tighter bounds on the width and convergence time for training GReLU networks, then show that they can be transformed to equivalent ReLU networks with $100\%$ train accuracy.
We further investigate the proposed technique and derive a  finite-width Neural Tangent Kernel equivalence with a improved condition on the network width. Finally, we empirically validate the effectiveness of our technique with datasets of different domains and practical network sizes.   


Our main contribution can be summarized as follows
\begin{itemize}
    \item We show that a randomly initialized deep ReLU network of depth $L$ and width $m=\tilde{\Omega}{\paren{n^2 L} }$ trained over $n$ samples of dimension $d$, reaches $\varepsilon$-error global minimum for $\ell_2$ regression in a logarithmic number of iterations, $T=c\log\paren{\frac{n^3L}{d\epsilon}}$.
    For comparison, the previous state-of-the-art result by~\cite{zou2019improved} required $m=\tilde{O}\paren{n^{8}L^{12}}$,
    $T=O\paren{n^2L^2\log(\frac{1}{\varepsilon})}$. 
    
    \item To achieve that, we propose a novel technique of training a surrogate  network with fixed activation patterns that can be transformed to its equivalent ReLU network at any point. 
    It is immune to bias-shifts and dying-ReLU problems and converges faster without affecting generalization, as we empirically demonstrate.
    \item We derive an NTK equivalence for GReLU networks with width \textit{linear} in $n$, connecting the regimes of practical networks and kernel methods for the first time (improving $\tilde{\Omega}\paren{n^32^{O(L)}}$ by \cite{huang2020dynamics}). This connection allows utilizing NTK-driven theoretic discoveries for practical neural networks, most notably generalization theory.
    
\end{itemize}
The remainder of this paper is organized as follows. In section~\ref{sec:preliminaries} we state the problem setup and introduce our technique.
Our main theoretical result is covered in section~\ref{sec:main_theory}, followed by additional properties of the proposed technique in Section~\ref{sec:properties}.
In section~\ref{sec:proof_sketch} we provide a proof sketch of the main theory, and
in section~\ref{sec:experiments} we empirically demonstrate its properties and effectiveness. 
Section~\ref{sec:related} contains additional related work. Finally, section~\ref{sec:conclusion} concludes our work.
\begin{table}[ht]
\centering
{\renewcommand{\arraystretch}{2}
  \begin{tabular}{p{1.3cm} P{2.3cm} P{2.8cm} P{1.4cm} P{2cm} P{4.15cm}}
    \hline
     & 
    $\tilde{\Omega}\left(\#\textrm{Neurons}\right)$ & 
    $O_{\varepsilon}\left(\#\textrm{Iters}\right)$ & $O\left(\textrm{Prob}\right)$ & $\tilde{\Theta}\left(\textrm{GD Step}\right)$ & 
    Remarks \\
    \hline
    Du\citeyearpar{du2019gradient} &
    $\frac{n^6}{\lambda_0^4p^3 }    $&
    $\frac{1}{\eta \lambda_0} \log\frac{1}{\varepsilon}    $ &
    \small p & 
    $\frac{\lambda_0}{n^2} $ & \small$\lambda_0^{-1}=\textrm{poly}\left(e^{L},n\right)$, binary -class, smooth activation  \\
    Zou\citeyearpar{zou2018stochastic} & 
    $ n^{26}L^{38}       $ &
    $  n^8L^9    $ &
    -- & $\frac{1}{ n^{29} L^{47}}    $ & binary-classification \\
    Allen-Zhu\citeyearpar{allen2019convergence} & 
    $n^{24}L^{12}$ & 
    $   n^6L^2 \log(\frac{1}{\varepsilon})    $ &
    $  e^{-\log^2 m}    $ &
    $   \frac{1}{n^{28}\log^5m L^{14}}  $ & $\propto\textrm{Poly}(\max_i|y_i|)$  \\
    Zou\citeyearpar{zou2019improved} & 
    $ n^{8}L^{12}       $ &
    $  n^2L^2 \log(\frac{1}{\varepsilon})$ &
    $n^{-1}$ & $\frac{1}{n^{8}L^{14}} $ & -- \\
    \hline
    \textbf{Ours} & 
    \small$ \mathbf{n^2 L}$ &
    \small
    $\mathbf{\log\paren{\frac{n^3L}{d_x\epsilon}}}$
    &
    \small$ \mathbf{e^{-\sqrt{m}}} $ &
    $\frac{d_x}{n^4 L^3 d_y}$
    & 
    $L=\Omega(\log n)$  
    \\
    \hline
  \end{tabular}
  }
  \caption{Comparison of leading works on overparameterized deep nonlinear neural networks trained with gradient-descent.} \label{table:theo_comparison}
\end{table}

\section{Preliminaries} \label{sec:preliminaries}
In this section we introduce the main problem and the novel
setup which is used to solve it. 
\\
\textbf{Notations:} 
we denote the Euclidean norm of vector $v$ by $\nm{v}$ or $\nm{v}_2$, and the Kronecker product and Frobenius inner-product for matrices $M,N$ by $M\otimes N$, $\langle M,N \rangle$ respectively. We denote by $\lambda_{\min}(M)$, $\lambda_{\max}(M)$ matrix minimal and maximal eigenvalues.  
We use the shorthand $[k]$ to denote the set $\{1,\dots,k\}$ and $A_{[k]}=\{A_1,\dots,A_k\}$.
We denote the positive part of a vector by $[v]^+=\paren{\max(0,v_1),\dots,\max(0,v_d)}$ and
its indicator vector by $[v]_+=(\mathbbm{1}_{v_1>0},\dots,\mathbbm{1}_{v_d>0})$. Matrix column-wise vectorization is denoted by $\mbox{vec}\left(\right)$. For two sequences $\{a_n\},\{b_n\}$, we denote $a_n=O(b_n)$ if there exist a constant $C_o$ such that $a_n\leq C_o b_n$, and $a_n=\Omega(b_n)$ if a constant $C_{\Omega}$ satisfies $a_n\geq C_{\Omega} b_n$. In addition, we denote $\tilde{O}(\cdot),\tilde{\Omega}(\cdot)$ to hide logarithmic factors.
Finally, we denote the normal and chi-square distributions by $\N(\mu, \Sigma),\chi^2_{d}$ respectively. 
\subsection{Problem Setup} \label{sec:setup} 
Let $\mathcal{T}=\braces{(x_i, y_i = \Phi_i x_i)}_{i\in[n]}$, be training examples, with $x_i\in \R^{d_x}, y_i \in \R^{d_y}, \Phi_i \in \R^{d_y\times d_x}$. Note that a different linear mapping $\Phi_i$ is used for each training example as our goal is to fit a nonlinear objective $y_i$. We assume for convenience and without loss of generality that the input data is normalized, $\nm{x_i}=1$. The data is used for training a fully-connected neural network with $L$ hidden layers and $m$ neurons on each layer.

We propose a novel network structure named \textbf{GReLU}, as illustrated in Figure~\ref{fig:network_blocks}, 
which improves the optimization, leading to tighter trainability bounds.
The key idea of this construction is to decouple the activation patterns from the weight tuning. The role of the activation patterns is to make sure different input samples propagate through different paths, while the role of the weights is to fit the output with high accuracy. 
Therefore, GReLU is composed of two parts, one consists of the weights that are trained, while the other defines the data-dependent activations of the ReLU gates. It is independently initialized and then remains fixed.
In Section \ref{sec:relu_equivalence}, we show that a GReLU network has an equivalent ReLU network, and those can be switched at any point. 
All layers are initialized with independent Gaussian distributions for $k\!\in\![L]$:
\\
\[
\brackets{W_k}_{i,j}\sim\N(0, 2/m)~~,
~~\brackets{\Psi_k}_{i,j}\sim\N(0, 2/m),
~~\brackets{C}_{i,j}\sim\N(0, 2/d_x),
~~\brackets{B}_{i,j}\sim\N(0, 2/d_y)
\]
Then, only the layers $W_{[L]}$ (blue blocks) are trained, while $\paren{B,C,\Psi_{[L]}}$ (red blocks) remain fixed during training. 
The activations $\{D^i_k, i \in [n], k\in [L]\}$ are computed based on the random weights of $\Psi_{[L]},B,C$ as follows,
\[
z^i_0 = [Cx_i]^+~~,~~z^i_k = [\Psi_k z_{k-1}^i]^+~~,~~  D_k^i = \mbox{diag}([z_k^i]_+)~~~k=1,\ldots,L.
\]
This implies that the activations do not change during training. 
Note that while previous works also used fixed initial and final layers
e.g.~\citep{allen2019convergence,zou2019improved}, the introduction of  $\Psi_{[L]}$ as fixed activation layers is novel. 

We denote the $k$th layer at iteration $t$ by $W_{t,k}\in \R^{m\times m}$, and the concatenation of all layers by $\bar{W}_t = (C,W_{t,1}, \ldots, W_{t,L},B)$. Since the activations are fixed in time and change per example, the full network applied on example $i$ is the following matrix,
\begin{eqnarray} \label{notation:parameter}
    W_t^i \vcentcolon= BD_L^iW_{t,L}\ldots D_1^iW_{t,1}D_0^iC\in \R^{d_y\times d_x}
\end{eqnarray}
Following previous works~\citep{allen2019convergence,du2019gradient,oymak2020towards}, we focus here for the sake of brevity on the task of regression with the square loss,
\begin{align}\label{notation:loss}
    \ell(\bar{W}_t) = \frac{1}{2}\sum_{i=1}^n \nm{(W_t^i - \Phi_i)x_i}^2.
\end{align} 
The loss is minimized by gradient-descent with a constant learning rate $\eta$.
Our proofs can be extended to other tasks and loss functions such as classification with cross-entropy, similarly to~\citep{allen2019convergence,zou2018stochastic} and omitted for readability.

Finally, our analysis requires two further definitions. The intermediate transform for $1\leq k' \leq k \leq L$ is defined as:
\begin{eqnarray}
Z^{t,i}_{k,k'} \vcentcolon= D_k^iW_{t,k}\ldots W_{t,k'+1}D_{k'}^i
\end{eqnarray}
and, the maximal variation from initialization of the network's trained layers is:
\begin{eqnarray} \label{def:tau}
    \tau \vcentcolon= \max\limits_{1\leq t \leq T}\max\limits_{k \in [L]} \left\|W_{t,k} - W_{1,k}\right\|_2.
\end{eqnarray}

\begin{figure}[t]    
          \centering
        \includegraphics[width=\linewidth]{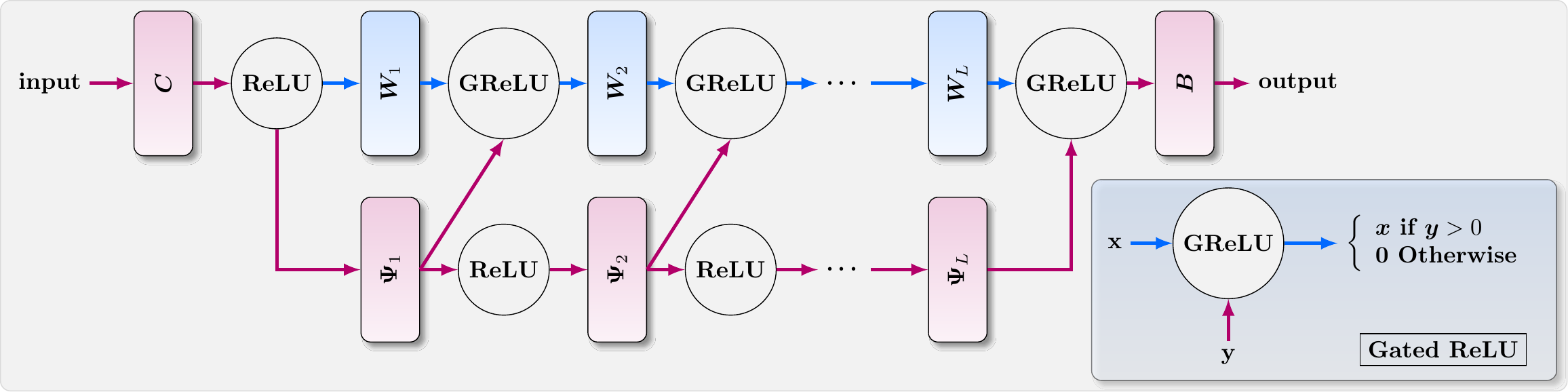}
        \caption{An illustration of the proposed network. Blue layers are trained while red layers set the activations and remain unchanged during training. 
        }
        \label{fig:network_blocks}
\end{figure}
\section{Main Theory} \label{sec:main_theory}
In this section, we present our main theoretical results. We start with making the following assumptions on the training data.
\begin{assumption}[non-degenerate input]
\label{assum:input}
Every two distinct examples $x_i,x_j$ satisfy  $\|x_i^\top x_j\|\leq \delta$. 
\end{assumption} 
\begin{assumption}[common regression labels] \label{assum:labels}
Labels are bounded: $\max_i|y_i|\leq \frac{m}{d_x} $.  
\end{assumption}
\noindent Both assumptions hold for typical datasets and are stated in order to simplify the derivation. 
Assumption~\ref{assum:labels} is more relaxed than corresponding assumptions in previous works
of labels bounded by a constant due to the overparameterization ($m\gg d_x$),
enabling the network handle large outputs better. 
We are ready to state the main result, guaranteeing a linear-rate convergence to a global minimum.
 
\begin{thm} \label{thm:linear_rate_convergence}
 Suppose a deep neural network of depth $L=\Omega(\log n)$ is trained by gradient-descent with learning rate 
 $\eta=\frac{d_x}{n^4 L^3 d_y}$
 under the scheme in Section~\ref{sec:setup}, and with a width $m$ that satisfies,
 \[
 m=\tilde{\Omega}{\paren{n^2 L d_y} }.
 \]
 Then, with a probability of at least $1 - \exp(-\Omega(\sqrt{m}))$ over the random initialization,
 the network reaches $\epsilon$-error within a number of iterations, 
 \[
 T=O\paren{\log\paren{\frac{n^3L}{d_x\epsilon}}}.
 \]
\end{thm}
\noindent
The proof appears in Section~\ref{proof:thm_linear_rate}. 
Theorem~\ref{thm:linear_rate_convergence} improves previous known results~\citep{zou2018stochastic,allen2019convergence,du2018gradient,zou2019improved} by orders of magnitude, as can be seen in Table~\ref{table:theo_comparison}. 
Specifically, it takes a step towards closing the gap between existing theory and practice, with required widths that can be used in practice for the first time.
The dependence of $m$ on the depth $L$ is linear, offering a significant improvement over the previous tightest bound of $O(L^{12})$~\citep{zou2019improved}. The bound on the number of iterations is logarithmic in the number of train examples, similarly to current common practice. For example, BiT~\citep[Sec.~3.3]{kolesnikov2019big} obtained SotA results on CIFAR-10/100 and ImageNet with a similar adjustment of the number of training epochs to the dataset size.
\\
An interesting question is whether the bounds we prove on $m$ and $T$ are \textit{tight}. The empirical evaluation presented in Figure~\ref{fig:train_loss_cifar} suggests they might be. We show that on a synthetic problem, fully-connected ReLU and GReLU networks with widths $m=\paren{n^2L}^p, p<1$, which is below our bound, do not converge to zero loss, while those with $p=1$ do. 
Next, we make two important remarks that further compare the above theoretical results with previous work.
\begin{remark}
The majority of works on the trainability of overparameterized networks focus on simplified shallow $1$-hidden layer networks (e.g.~\citep{zou2019improved,du2018gradient,wu2019global,oymak2020towards,song2019quadratic,li2018learning}). Our focus is on the more challenging practical deep networks, and Theorem~\ref{thm:linear_rate_convergence} requires a minimal depth $L=\Omega(\log{n})$. Our analysis, however, can be extended to $1$-hidden layer networks with the same dependencies on the number of train samples, i.e., $m=\tilde{\Omega}(n^2), ~T=O\paren{\log n}$. 
According to our knowledge, these dependencies improve all previous results on shallow networks.
\end{remark}

\begin{remark}
For most practical regression tasks $d_y$ is small, therefore often assumed to be of $O(1)$ and ignored in previous works. We make no such assumption throughout this work and specifically in Theorem~\ref{thm:linear_rate_convergence}, but omit it from Table~\ref{table:theo_comparison} for a clear comparison. 
\end{remark}
\noindent
The proposed GReLU construction is used here to prove bounds on optimization. However, we believe it could be used as a technical tool for better understanding the generalization of deep neural networks. Previous works offered convergence guarantees on networks with practically infinite width.
Empirical results show a poor generalization ability for such networks~\citep{arora2019exact,li2019enhanced} when compared to mildly overparameterized ones~\citep{srivastava2015training,noy2020asap,nayman2019xnas} over various datasets. Tighter bounds on the network size are possibly key to \textit{capture the training dynamics} of deep neural networks, and to further improve those from a theoretic perspective. We leave this to future work.
\section{Properties} \label{sec:properties}
In this section, we discuss the properties of the suggested technique. It complements the former section by providing intuition on \textit{why} we are able to achieve the rates in Theorem~\ref{thm:linear_rate_convergence}, which are more similar to sizes used in practice.

First, it is important to note that modern state-of-the-art models are not based on fully-connected networks but rather on improved architectures.
A classic example is the improved optimization achieved by residual connections~\citep{he2016deep}.
Neural networks equipped with residual connections typically retain an improved empirical performance for a relatively small degradation in runtime~\citep{monti2018avoiding}. However, Previous theoretical works on trainability of residual networks
do not share those improvements~\citep{allen2019convergence,du2018gradient}. 
Few works have shown that a residual network can be transformed to its equivalent network with no residual connections (e.g. ~\citep{monti2018avoiding}).

Similarly to residual networks, GReLU networks offer appealing properties regarding the optimization of deep networks (Section~\ref{sec:improved_opt}), and can be mapped to their equivalent ReLU networks at any time (Section~\ref{sec:relu_equivalence}).
Unlike residual networks, GReLU networks are based on theoretical guarantees (Sections~\ref{sec:main_theory}, \ref{sec:NTK}), and actually accelerate the network runtime as discussed in Section~\ref{sec:implementation}.
\subsection{ReLU Network Equivalence} \label{sec:relu_equivalence}
ReLU networks are extensively studied theoretically due to their dominating
empirical performance regarding both optimization and generalization. While this work studies optimization, the following equivalence between GReLU and ReLU networks is important to infer on the generalization ability of the proposed technique. 
\begin{thm} \label{thm:equivalence2}
Let $\bar{W}_t = (W_{t,1}, \ldots, W_{t,L};C,B,\Psi_{[L]})$ be an overparameterized GReLU network of depth $L$ and width $m$, trained for $t$ steps.
Then, a unique equivalent ReLU network of the same sizes $\tilde{W}_t = (\tilde{W}_{t,1}, \ldots, \tilde{W}_{t,L};C,B)$ can be obtained, with identical intermediate and output values over the train set. 
\end{thm}
\noindent The proof appears in Section~\ref{proof:relu_equivalence} and is done by construction of the equivalent ReLU network. The equivalent networks have an identical training footprint, in the sense that every intermediate feature map of any train example is identical for both, leading to identical predictions and training loss.
This offers a dual view on GReLUs: a technique to improve and accelerate the training of ReLU networks. A good analogy is Batch Normalization, used to smooth the optimization landscape~\citep{santurkar2018does}, and can be absorbed at any time into the previous layers~\citep[(20)-(21)]{krishnamoorthi2018quantizing}
Notice that a ReLU network can be also trained as a GReLU network at any time by fixing its activation pattern, and using the scheme in section~\ref{sec:setup} with $\Psi_k\leftarrow W_k, k\in[L]$.

We note that while our proof uses the Moore–Penrose pseudo inverse for clarity, modern least-squares solvers~\citep{lacotte2019faster} provide similar results more efficiently, so the overall complexity of our method stays the same. In other words, training according to Theorem~\ref{thm:linear_rate_convergence} and then switching to the equivalent ReLU network results in a trained ReLU network with an $\varepsilon$-error in $T=O\paren{\log\paren{\frac{n^3L}{d_x\epsilon}}\cdot m^2L}$ operations, as $m^2L$ stands for the number of trained network parameters.
\subsection{Neural Tangent Kernel Equivalence} \label{sec:NTK}
The introduction of the Neural Tangent Kernel (NTK)~\citep{jacot2018neural} offered a new perspective on deep learning, with insights on different aspects of neural networks training.  
Followup works analyzed the behavior of neural networks in the \textit{infinite-width regime}, where neural networks roughly become equivalent to linear models with the NTK. 
An important result by \cite{allen2019convergence}, extended this equivalence to \textit{finite-width} networks for $m=\Omega\paren{Poly(n,L)}$. \cite{huang2020dynamics} reduced the sample-size dependency to  $m=\tilde{\Omega}\paren{n^32^{O(L)}}$, still prohibitively large for practical use.

In this section, we derive a corresponding finite-width NTK equivalence, with a width that is \textit{linear} in the sample-size. Closing this gap between theory and practice offers advantages for both:
\begin{enumerate}
    \item Practice: networks can be actually trained in the NTK regime, optimizing weights and architectures to
    retain properties such as improved generalization \citep{geiger2020scaling,xiao2019disentangling} and overall performance~\citep[Sec 3.1]{lee2020finite}.  
    \item Theory: NTK-driven theoretic insights can be utilized effectively for neural networks of practical sizes. Those insights relate to loss landscapes, implicit regularization, training dynamics and generalization theory~\citep{belkin2018understand,cao2019generalization,belkin2018overfitting,li2019enhanced,huang2020deep,liang2020just}. 
\end{enumerate}



\noindent The analysis of Theorem~\ref{thm:linear_rate_convergence} bounds the maximal variation of the network's layers with a probability of at least $1 - \exp(-\Omega(\sqrt{m}))$ as follows,
\begin{eqnarray} \label{def:xi_ntk}
\tau = \max\limits_{1 \leq t \leq T}\max\limits_{k\in [L]} \nm{W_{t,k}-W_{1,k}}_2 
\leq \tilde{O}\left(\frac{d_yn^{1/2}m^{-1/2}}{L}\right)
 \vcentcolon= \frac{\xi}{L}
\end{eqnarray}
for some small number $\xi(m;n,d_y)$. Let $\bar{W} = (\bar{W}_{1}, \ldots, \bar{W}_{L})$ be any solution which satisfies,
\begin{eqnarray*}
\nm{\bar{W}_{k}-W_{1,k}}_2\leq\xi L^{-1}~,~\forall {k\in [L]}
\end{eqnarray*}
Denote the difference $W'=\bar{W}-W_1$, and notice that $\nm{W'}_2\leq \xi$.\\
Define the NTK and the NTK objective~\citep{jacot2018neural,allen2019convergence} of the initialized network for the $p$-th output, $p\in[d_y]$,  respectively,
\begin{eqnarray}
K_p^{\textrm{NTK}}(x_i, x_j) \vcentcolon= \langle \nabla y_p(x_i,W_1),\nabla y_p(x_j,W_1)\rangle \quad,\quad 
y_p^{\textrm{NTK}}(x_i,W')=\langle \nabla y_p(x_i,W_1),W' \rangle
\end{eqnarray}
\cite{jacot2018neural}  proved that for an \textit{infinite} $m$, the dynamic NTK and NTK are in fact equivalent as $\xi \rightarrow 0$. Allen-Zhu  et al.~\citep{allen2019convergence} showed a high-order polynomial bound on this equivalence. We further improve their results by stating a tighter bound on $m$ for our setting.
\begin{thm} \label{thm:NTK}
For every $x_i\in \R^{d_x}$, $1\leq p \leq d_y$, $W': \nm{W'}_2\leq \xi$ and $m=\Omega(n d_x d_y^2 L^3)$, with probability of at least $1-\exp(-\Omega(m))$ over the initialization of $\braces{B,C,\Psi,W_1}$ we have,
\begin{eqnarray}
\nm{\nabla y_p(x_i,W_1+W')-\nabla y_p^{\textrm{NTK}}(x_i,W')}_F \leq \mathcal{R} \nm{\nabla y_p^{\textrm{NTK}}(x_i,W')}_F
\end{eqnarray}
with the ratio
\[\mathcal{R}= \tilde{O}\left(\frac{n^{1/2}d_x(L d_y)^2 }{m^{5/2}}\right).
\]
\end{thm}
\begin{cor} \label{thm:NTK_corollary}
For every $x_i,x_j\in \R^{d_x}$, $1\leq p \leq d_y$, $W': \nm{W'}_2\leq \xi$ and $m=\Omega(n d_x d_y^2 L^3)$ ,  with probability of at least $1-\exp(-\Omega(m))$,
\begin{eqnarray} \label{eqn:NTK_cor_main}
|\langle \nabla y_p(x_i,W_1+W'),\nabla y_p(x_j,W_1+W') \rangle- K_p^{\textrm{NTK}}(x_i, x_j) | \leq 3\mathcal{R} \sqrt{K_p^{\textrm{NTK}}(x_i, x_i) K_p^{\textrm{NTK}}(x_j, x_j)}
\end{eqnarray}
\end{cor}
\noindent Proofs can be found in Sections~(\ref{proof:thm_NTK}-\ref{proof:thm_NTK_cor}). Theorem~\ref{thm:NTK} guarantees that the difference between gradients is negligible compared to their norms, thus the
NTK is a good approximation for the dynamic one. For simplicity, we state the result per single output dimension $p$, while it holds for outputs of dimension $d_y$ as well.
\\
For comparison, the corresponding ratio by \cite{allen2019convergence} is significantly worse: $\frac{L^{3/2}}{\sqrt{\log m}}$.
More Specifically, equipped with the width required by Theorem~\ref{thm:linear_rate_convergence}, the guaranteed ratio by Theorem~\ref{thm:NTK} is negligible even for small problems $\mathcal{R}^{-1}=\tilde{O}\paren{n^{9/2}d_x^{-1}(d_y L)^{1/2}}$.
The more regular trajectory of GReLU gradients compared to ReLU gradients
 can be also seen in experiments with smaller network widths, as demonstrated in Figure~\ref{fig:train_loss_grad}.
 
\subsection{Improved Gradient Optimization} \label{sec:improved_opt}
Modern state-of-the-art models are \textit{modified versions} of fully-connected networks, offering improved optimization 
with additional layers~\citep{he2016identity,ioffe2015batch,bachlechner2020rezero}, activations~\citep{clevert2015fast,xu2015empirical} and training schemes~\citep{zagoruyko2017diracnets,jia2017improving}.
Those are mostly aimed at solving problems that arise when training plain deep ReLU networks. It is important to note that ReLU networks often do not achieve small training error without such modifications.
\\ 
Two of the most studied problems of ReLU networks are bias shifts (mean shifts)~\citep{clevert2015fast} and dying ReLUs~\citep{lu2019dying}.\\
\textbf{Bias shifts} affect networks equipped with activations of non-zero mean. Specifically, ReLU is a non-negative function, 
leading to increasingly positively-biased layer outputs in ReLU networks, a problem that grows with depth, $\textrm{E}_{W_k}\brackets{\max({W_k z_{k-1},0})}-z_{k-1}>0$.
Reduced bias shifts were shown to speed up the learning by bringing the normal gradient closer to the unit natural gradient~\citep{clevert2015fast}.
Interestingly, GReLUs do not suffer from bias shifts, as negative values are passed as well, and the fixed activations of different layers are independent.
\\
\textbf{Dying ReLUs} refers to the problem of ReLU neurons become inactive and only output $0$ for any input. This problem also grows with depth, as Lu et al.~\citep{lu2019dying} show that a deep enough ReLU network will eventually become a constant function. \\
Considering a GReLU network, given the initial network is properly initialized (i.e., not 'born dead'), it is guaranteed that no neurons will die during training due to its fixed activation pattern. 
\begin{remark}
GReLU networks are immune to the problems of bias-shifts and dying-ReLUs.
\end{remark}
\noindent These properties essentially lead to better back-propagation of gradients during training and corresponding faster convergence to global minima under milder requirements on the networks' depths. 
\subsection{Faster Train and Inference Iterations} \label{sec:implementation}
Theorem~\ref{thm:linear_rate_convergence} guarantees a faster convergence to global minima in terms of the number of iterations. We now explain how a straightforward implementation of a GReLU network leads to an additional approximated $2\times$ shorter \textit{iteration time} compared to its equivalent ReLU network, for both train and inference. \\
\textbf{Train:} consider a single GReLU train step as described in Figure~\ref{fig:network_blocks}. Since the activations are fixed, a binary lookup table of size $n\times m\times L$ is made once and used to calculate only the values related to active ReLU entries, which are around $50\%$ due to the proposed initialization of $\Psi_{[L]}$, leading to approximately $50\%$ of the FLOPS of a fully-connected network of the same dimensions\footnote{Acceleration is achieved without additional implementation, as Pytorch and Tensorflow automatically skip backward-propagation calculations of such zeroed gradients.}. \\
\textbf{Inference:} A non-negative matrix approximation (NMA)~\citep{tandon2010sparse} is used to replace each matrix $\Psi_k$ with smaller ones with respective sizes $(m\times r)$ and $(r\times m)$, $\Psi_k\approx W_k H_k$ for $r<<m$. The total FLOPS count is therefore proportional to $O(\mathbf{0.5}m^2+2rm)\approx O(\mathbf{0.5}m^2)$. While this approximation is not covered by our theory, experiments with $r=\sqrt{m}/2$ yielded the desired acceleration without any accuracy degradation. This result leads to a general insight.
\begin{remark} \label{remark:acceleration}
Our technique for $2\times$ accelerated inference can be used to accelerate any neural network with ReLU activation. Given network layers $W_{[L]}$, simply transition to the GReLU architecture described in Figure~\ref{fig:network_blocks} with, $\Psi'_k\leftarrow W_k, \forall k\in[L]$, then use NMA approximation, $\Psi'_k\approx W_k H_k$. \\
Additional modifications can be applied for further acceleration with minimal inference variation, like binarized $\Psi'_{L}$ matrices~\citep{hubara2016binarized}. Those exceed our scope and are left as future work.
\end{remark}

\section{Main Theory Proof Sketch} \label{sec:proof_sketch}
Calculating the gradient of $\ell(W_t)$ over $W_{t,k}$, denoted by $\nabla_k\ell(W_t)$, is straightforward,
\begin{eqnarray} 
\nabla_k \ell(W_t) = \sum_{i=1}^n \left[F^i_{t,k+1}\right]^{\top}(W_t^i - \Phi_i)x_ix_i^{\top} \left[G^i_{t,k-1}\right]^{\top} \in\R^{m\times m} \label{eqn:grad}
\end{eqnarray}
where
\begin{align}\label{notation:gradient:parameter}
F^i_{t,k} = BD_L^iW_{t,L}\ldots D_k^iW_{t,k}D_{k-1}^i\in\R^{d_y\times m}, \quad G^{i}_{t,k} = D_k^iW_{t,k}\ldots D_1^iW_{t,1}D_0^iC \in\R^{m\times d_x}
\end{align} 
and we set $G_{t,0}^i = D_0^iC$. Those represent the network partition and will be referred to as sub-networks. Notice that $W^i_t=F^i_{t,k+1}G^i_{t,k}$. We now calculate the difference between $W_{t+1}^i$ and $W_t^i$.
\begin{lemma}\label{lemma:diff:W}
\begin{align*}
    W_{t+1}^i - W_t^i = -\eta \sum_{k=1}^L F_{t,k+1}^i\left[F_{t,k+1}^i\right]^{\top}\left(W_t^i - \Phi_i\right)x_ix_i^{\top}[G_{t,k-1}^i]^{\top}G_{t,k-1}^i -\eta\Gamma_{t,i} + \eta^2 \Delta_{t,i},
\end{align*}
where
\begin{align}
\Gamma_{t,i} := & \sum_{k=1}^L \sum_{j\neq i}F_{t,k+1}^i\left[F_{t,k+1}^j\right]^{\top}\left(W_t^j - \Phi_j\right)x_jx_j^{\top}\left[G_{t,k-1}^j\right]^{\top}G_{t,k-1}^i. \label{eqn:gamma} \\
\Delta_{t,i} := & \sum_{s=2}^L(-\eta)^{s-2}\sum_{L\ge k_1>k_2 \ldots> k_s \ge 1}F^i_{t,k_1+1}\nabla_{k_1}\ell(W_t)D_{k_1-1}^iW_{t,k_1-1}\ldots D_{k_s}^i\nabla_{k_s}\ell(W_t)G^i_{t,k_{s}-1}, \label{eqn:delta}
\end{align}
\end{lemma}
\noindent Lemma~\ref{lemma:diff:W} breaks a single gradient step into $3$ terms. The first term on the left represents the impact of  example $i$ onto itself, where the next term~(\ref{eqn:gamma}) represents its impact by other examples and has a key role in our analysis. Our technique of fixed activations enables a more careful optimization of this term, followed by improved guarantees.
The last term~(\ref{eqn:delta}) stands for higher-order dependencies between different weights, i.e. the impact of a change in some weights over the change of others, as a result of optimizing all weights in parallel. It negatively affects the optimization and should be minimized. The next lemma deals with the loss change.
\begin{lemma}\label{lemm:descent}
For any set of positive numbers $a_1, \ldots, a_n$, we have
\begin{eqnarray} \label{eqn:losses_diff}
\ell(W_{t+1}) - \ell(W_t) \leq \sum_{i=1}^n \frac{\Lambda_i + \eta^2 a_i}{2}\nm{(W_t^i - \Phi_i)x_i}^2 + \sum_{i=1}^n\frac{\eta^2\left(3\eta^2 + 1/a_i\right)}{2}\nm{\Delta_{t,i}x_i}^2
\end{eqnarray}
where
\begin{align*} 
\Lambda_i = & -2\eta \sum_{k=1}^L \lambda_{\min}\left(F_{t,k+1}^i\left[F_{t,k+1}^i\right]^{\top}\right)\lambda_{\min}\left(\left[G_{t,k-1}^i\right]^{\top}G_{t,k-1}^i\right) \\
&  + 2\eta  \sum_{k=1}^L\sum_{j \neq i} \left|\langle G^j_{t,k-1}x_j, G^i_{t,k-1}x_i\rangle\right| \nm{F^j_{t,k+1}\left[F_{t,k+1}^i\right]^{\top}}_2 \\
& + 3\eta^2L\sum_{k=1}^L \lambda_{\max}\left(F_{t,k+1}^i\left[F_{t,k+1}^i\right]^{\top}\right) \nm{G_{t,k-1}^ix_i}^4 \\
&+  3\eta^2nL \sum_{k=1}^L\sum_{j \neq i} \left|\langle G^j_{t,k-1}x_j, G^i_{t,k-1}x_i\rangle\right|^2 \nm{F^j_{t,k+1}\left[F_{t,k+1}^i\right]^{\top}}_2^2.
\end{align*}
\end{lemma}
\noindent We wish to bound the right-hand side with a value as negative as possible. The first term of~(\ref{eqn:losses_diff}) is proportional to the loss~(\ref{notation:loss}). Thus negative $\Lambda_i+\eta^2a_i$ values can lead to an exponential loss decay, i.e. linear-rate convergence: $\ell(W_{t+1})=(1-|\rho|)\ell(W_t)$. 
In order to minimize $\Lambda_i$, two properties are desired:
(a) The eigenvalues of the squared weight matrices $F^i_{t,k}\left[F_{t,k}^i\right]^{\top}, \left[G_{t,k}^i\right]^{\top}G_{t,k}^i$ are concentrated, that is the ratio between the minimal and maximal eigenvalues is independent of the problem's parameters. (b) The covariance between sub-networks of different examples is much smaller than the covariance of the same examples.
The second term of~(\ref{eqn:losses_diff}) represents high-order dependencies between gradients and is positive, restricting the learning rate from being too high. \\
Next, we assume the following inequalities which satisfy (a),(b) hold with a high probability,
\begin{eqnarray}
&  & \lambda_{\min}\left(F_{t,k}^i\left[F_{t,k}^i\right]^{\top}\right) \geq \alpha_y, \quad \lambda_{\min}\left(G_{t,k}^i\left[G_{t,k}^i\right]^{\top}\right) \geq \alpha_x, \label{eqn:eig-lower-bound} \\
&  &  \lambda_{\max}\left(F_{t,k}^i\left[F_{t,k}^i\right]^{\top}\right) \leq \beta_y, \quad \lambda_{\max}\left(G_{t,k}^i\left[G_{t,k}^i\right]^{\top}\right) \leq \beta_x, \label{eqn:eig-upper-bound} \\
&  & \left|\left\langle G_{t,k-1}^jx_j, G_{t,k-1}^ix_i\right\rangle\right|\nm{F^j_{t,k+1}\left[F_{t,k+1}^i\right]^{\top}}_2 \leq \gamma\beta^2, \label{eqn:cross-upper-bound} \\
&  & \beta^2\gamma n\leq \frac{\alpha^2}{2}, \label{eqn:gamma-upper-bound}
\end{eqnarray}
where $\alpha = \sqrt{\alpha_x\alpha_y}$, $\beta = \sqrt{\beta_y\beta_x}$ and $\gamma$ are auxiliary parameters.
Assumptions (\ref{eqn:eig-lower-bound},\ref{eqn:eig-upper-bound}) are related to property (a), while (\ref{eqn:cross-upper-bound}, \ref{eqn:gamma-upper-bound}) guarantee property (b).
We will show that the above bounds hold with a high probability under the proposed initialization and the fixed activation patterns defined in Section~\ref{sec:preliminaries}. Using those assumptions, we state the following lemma.
\begin{lemma}\label{lemma:linear_rate}
Set $a_i = \beta^4L^2$ and
\begin{eqnarray}
    \eta = \min\left( \frac{ \alpha^2}{12\beta^2\beta_xL}, \frac{1}{3 L}, \frac{\alpha^2}{ 4\beta^4 L}, \frac{1}{ \beta^2 L}, \frac{1}{4\sqrt{2}\sqrt{L}  e^{\theta/2}\theta^{-1/2}\beta \sqrt{\ell_0}}, \frac{\alpha^2}{1024 n e^{2\theta}\theta^{-2}   \beta^2\ell_0}\right). \label{eqn:eta_min}
\end{eqnarray}
Under assumptions in (\ref{eqn:eig-lower-bound}), (\ref{eqn:eig-upper-bound}), (\ref{eqn:cross-upper-bound}), (\ref{eqn:gamma-upper-bound}) and $\ell(W_t) \leq \ell_0$ for $t \geq 1$, for any $\theta \in (0, 1/5)$, with a probability $1 - L^2\sqrt{m}\exp\left(-\theta m/[4L] + 6\sqrt{m}\right)$, we have
\begin{eqnarray} \label{eqn:linear_rate}
        \ell(W_{t+1}) \leq \paren{1-\frac{\eta\alpha^2 L}{2}}\ell(W_t).\label{eqn:bound-1}
\end{eqnarray}
\end{lemma}
\noindent This inequality  indicates a linear-rate convergence with rate $\frac{\eta\alpha^2 L}{2}$. The proof of Theorem~\ref{thm:linear_rate_convergence} follows directly from~(\ref{eqn:bound-1}). Proofs for the validity of the above assumptions with high probability are based on concentration bounds for sufficiently overparameterized networks under the unique initialization proposed in Section~\ref{sec:preliminaries}, and are quite technical. \\
 Theorem~\ref{thm:bound_alpha_beta} validates (\ref{eqn:eig-lower-bound},\ref{eqn:eig-upper-bound}) by bounding the ratio between the maximal and minimal eigenvalues of the squared weight matrices with a small constant for all $(t,k,i)$. Theorem~\ref{thm:bound_gamma} utilizes the fixed activation patterns to show a small covariance between sub-networks of different examples as required by (\ref{eqn:cross-upper-bound},\ref{eqn:gamma-upper-bound}).   
 Both theorems hold for small enough maximal variation from initialization $\tau$~(\ref{def:tau}), which is achieved
 by sufficient network overparameterization, as shown in Theorem~\ref{thm:linear_rate_convergence}. 

\section{Experiments} \label{sec:experiments}
In this section, we provide empirical validation of our theoretical results, as well as further testing on complementary setups. We wish to answer the following questions:
\begin{enumerate}
    \item \textit{Is the theoretical guarantee of Theorem~\ref{thm:linear_rate_convergence} tight, or can it be further improved?}
        \item \textit{How well do GReLU networks train and generalize compared to ReLU networks?}
    \item \textit{What is the difference between the training dynamics of GReLU and ReLU when optimized with gradient-descent?}
\end{enumerate}
The next sections are dedicated to discuss these questions.
For fairness, we compare networks with the same width and depth, initialization described in Section~\ref{sec:preliminaries}, and trained by gradient descent with a constant learning rate for each train. The term MSE refers to the \textit{sum} of squared errors described in~(\ref{notation:loss}).
Since the theory covers general setups, it often uses small learning rate values to fit all setups, while in practice larger values can be utilized for faster convergence. Therefore, we experiment with a variety of learning rates according to the task. In addition, in all the experiments we optimize with batch gradient-descent due to hardware constraints. We pick the largest possible batch size per experiment.

\begin{figure}[t]
    \centering
     \begin{subfigure}[b]{0.35\textwidth}
        \centering
     \includegraphics[width=\textwidth]{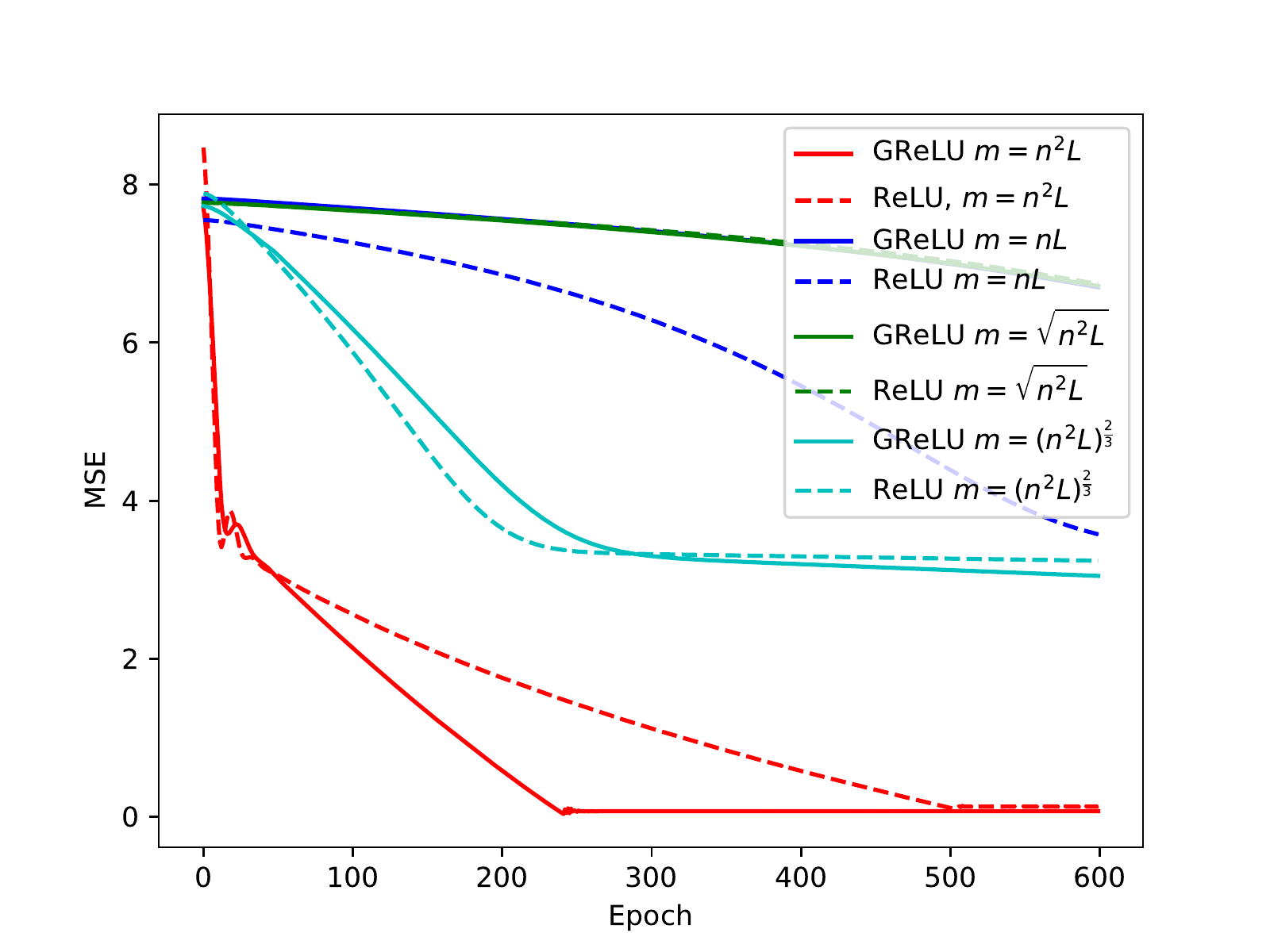}
    \caption{}
    \label{fig:theory_graph}
     \end{subfigure} \hspace{-0.8cm}
     \begin{subfigure}[b]{0.35\textwidth}
         \centering
    \includegraphics[width=\textwidth]{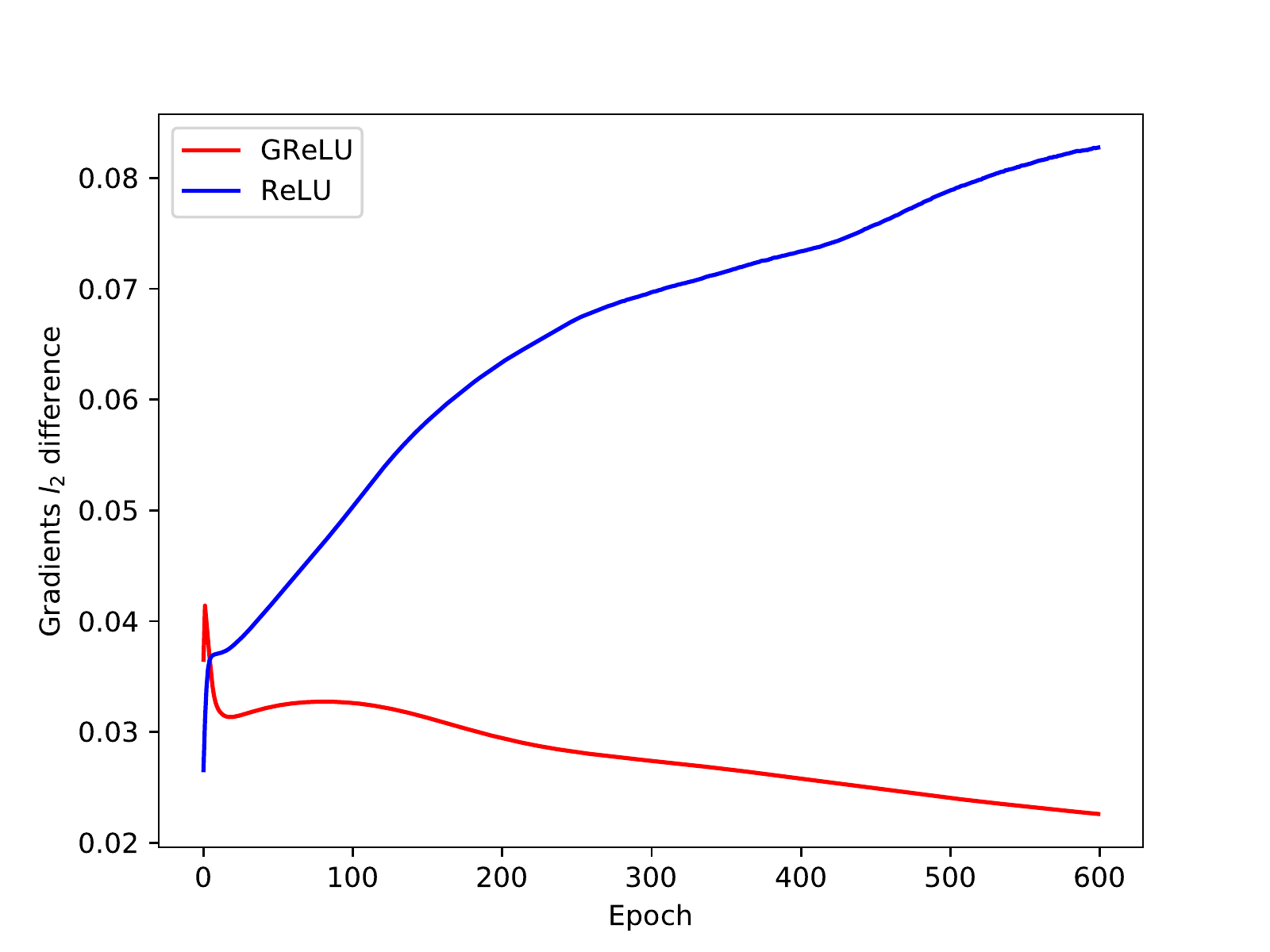}
    \caption{}
    \label{fig:train_loss_grad}
    \end{subfigure} \hspace{-0.8cm}
    \begin{subfigure}[b]{0.35\textwidth}
         \centering
    \includegraphics[width=\textwidth]{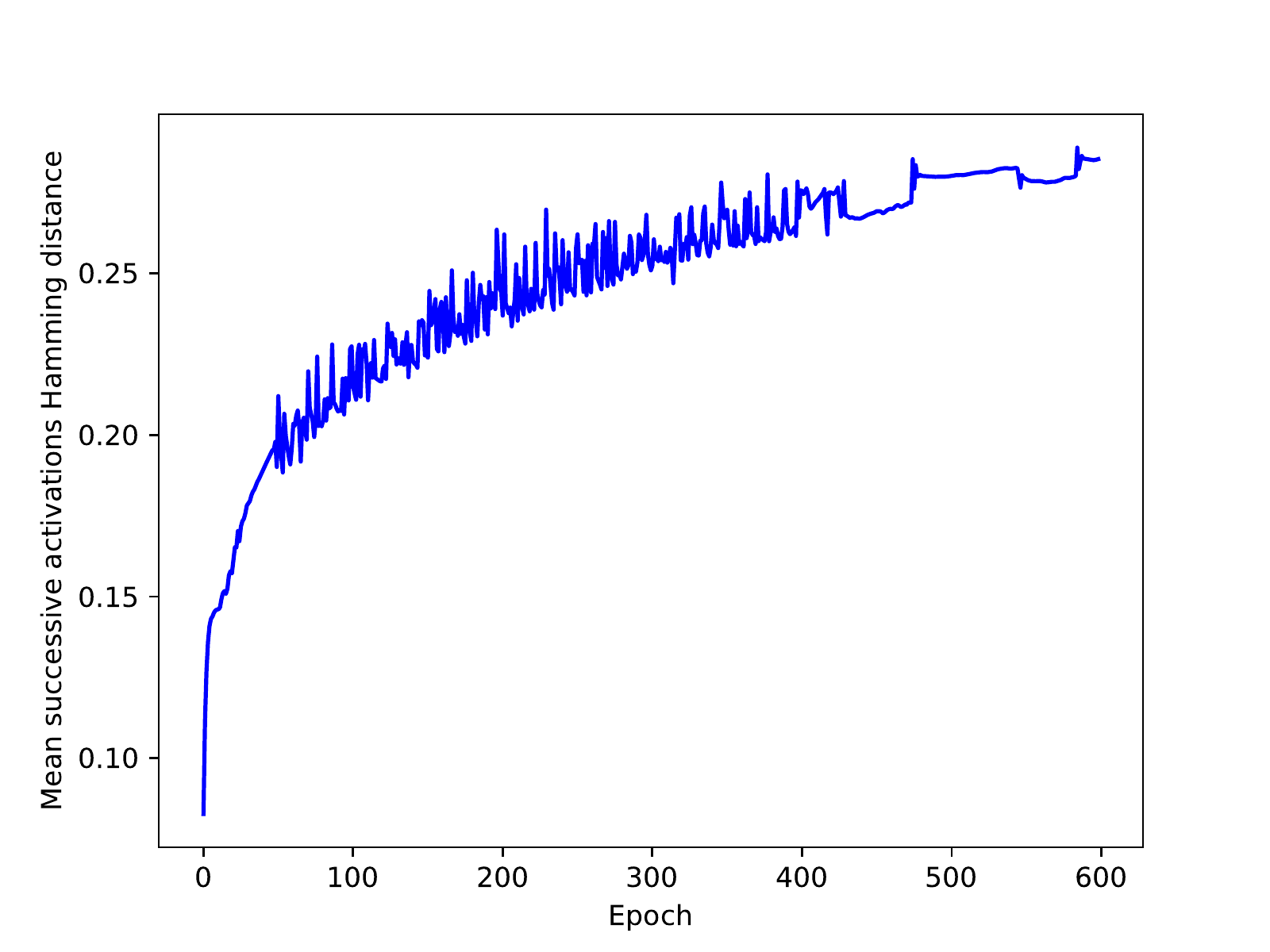}
    \caption{}
    \label{fig:successive_relu_hamming}
    \end{subfigure} 
     \caption{Training dynamics of ReLU and GReLU networks. (a) ReLU (solid lines) and ReLU (dashed) networks with different widths trained over synthetic dataset. (b) Average temporal difference of gradients on SGEMM dataset. (c) Average Hamming distance of successive ReLU activation patterns for SGEMM dataset.}
\end{figure}

\subsection{Lower Trainability Bound}
To answer \textbf{question 1}, we experiment on a synthetic dataset, in order to control the number of examples and their dimension to fit the theory under GPU memory constraints. 
We generate a regression problem with the following complex version of the Ackley function~\citep{potter1994cooperative},
\begin{eqnarray*}
y_i=\sum_{d'=1}^d x_{i,d-d'}\left(\log{\left(|x_{i,d'}|\right)\left(\cos(x_{i,d'})+x^3\sin(x_{i,d'})\right)}+\sqrt{|x_{i,d'}|}\right)~~,~~i=1,\dots, n
\end{eqnarray*}
We generate $\paren{n,d}=\paren{100,200}$ random vectors and labels, and use $4$-layers networks for both ReLU and GReLU, with various widths. Each experiment is repeated 5 times with different random seeds.

The results of the different runs appear in Figure~\ref{fig:theory_graph}. To have a fair comparison between the different settings, we proceed as follows. For the experiment corresponding to the theoretical width in Theorem~\ref{thm:linear_rate_convergence}, we plot a curve where every entry is computed by taking the \textit{maximal} MSE over the 5 runs. 
For lower widths, $m=\paren{n^2L}^p, p<1$ and $m=n L$, we plot the \textit{minimal} MSE over the corresponding runs. 
All $5$ runs with our theoretical width converged to zero error, for both ReLU and GReLU. 
In contrast, none of the runs with lower width converged, for both networks. Since the theory guarantees convergence for \textit{any data distribution}, these results provide some empirical support that Theorem~\ref{thm:linear_rate_convergence} is tight. 

We acknowledge that training with significantly smaller learning rates (e.g., $10^{-20}$) would correspond to longer and impractical convergence times, but possibly would allow narrower networks to converge as well. In addition, while we focus on fully-connected networks, 
modern architectures like ResNets were empirically shown to converge with smaller widths (channels) for specific tasks. Improved theoretic guarantees for these architectures via GReLU technique are interesting for future work.  

\subsection{Optimization and Generalization in Practical Scenarios}
\begin{figure}[h]
    \centering
    \begin{subfigure}[b]{0.35\textwidth}
        \includegraphics[width=\textwidth]{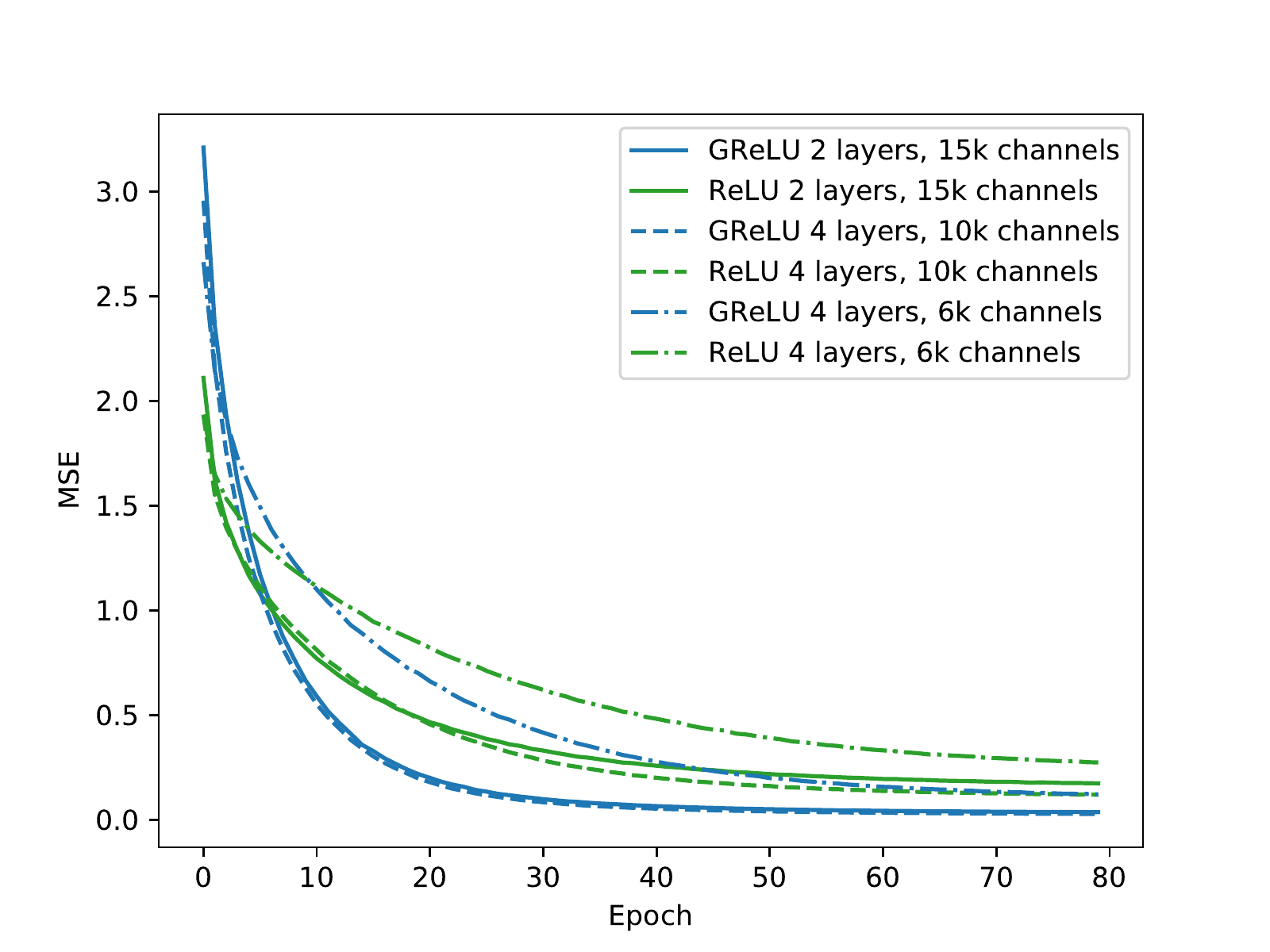}
        \caption{Train MSE vs epoch}
        \label{fig:train_loss_cifar}
    \end{subfigure} \hspace{-0.8cm}
     \begin{subfigure}[b]{0.35\textwidth}
         \centering
            \includegraphics[width=\textwidth]{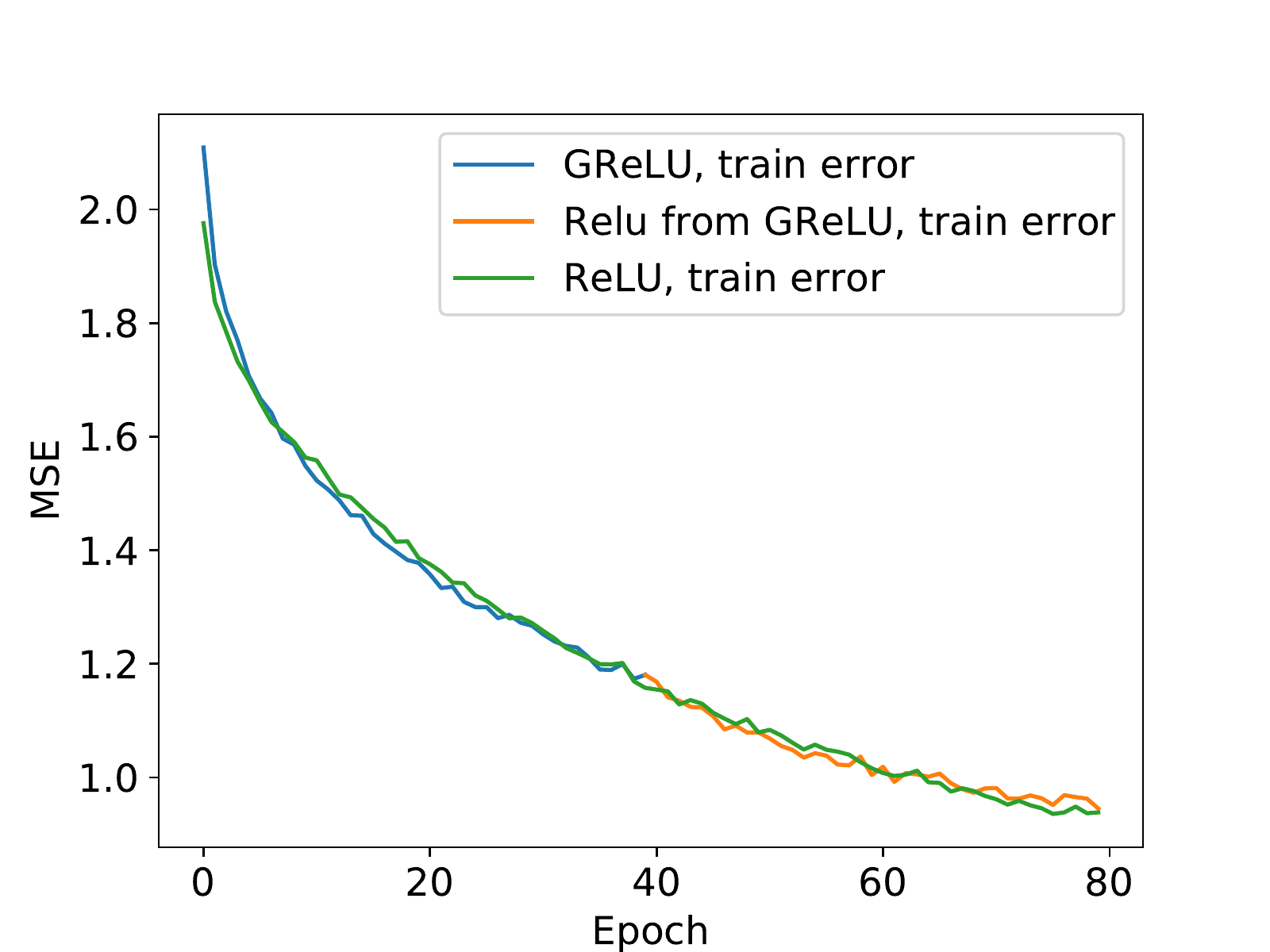}
            \caption{Train MSE vs epoch}
        \label{fig:train_loss_grelu_relu_cifar}
    \end{subfigure} \hspace{-0.8cm}
    \begin{subfigure}[b]{0.35\textwidth}
         \centering
            \centering
    \includegraphics[width=\textwidth]{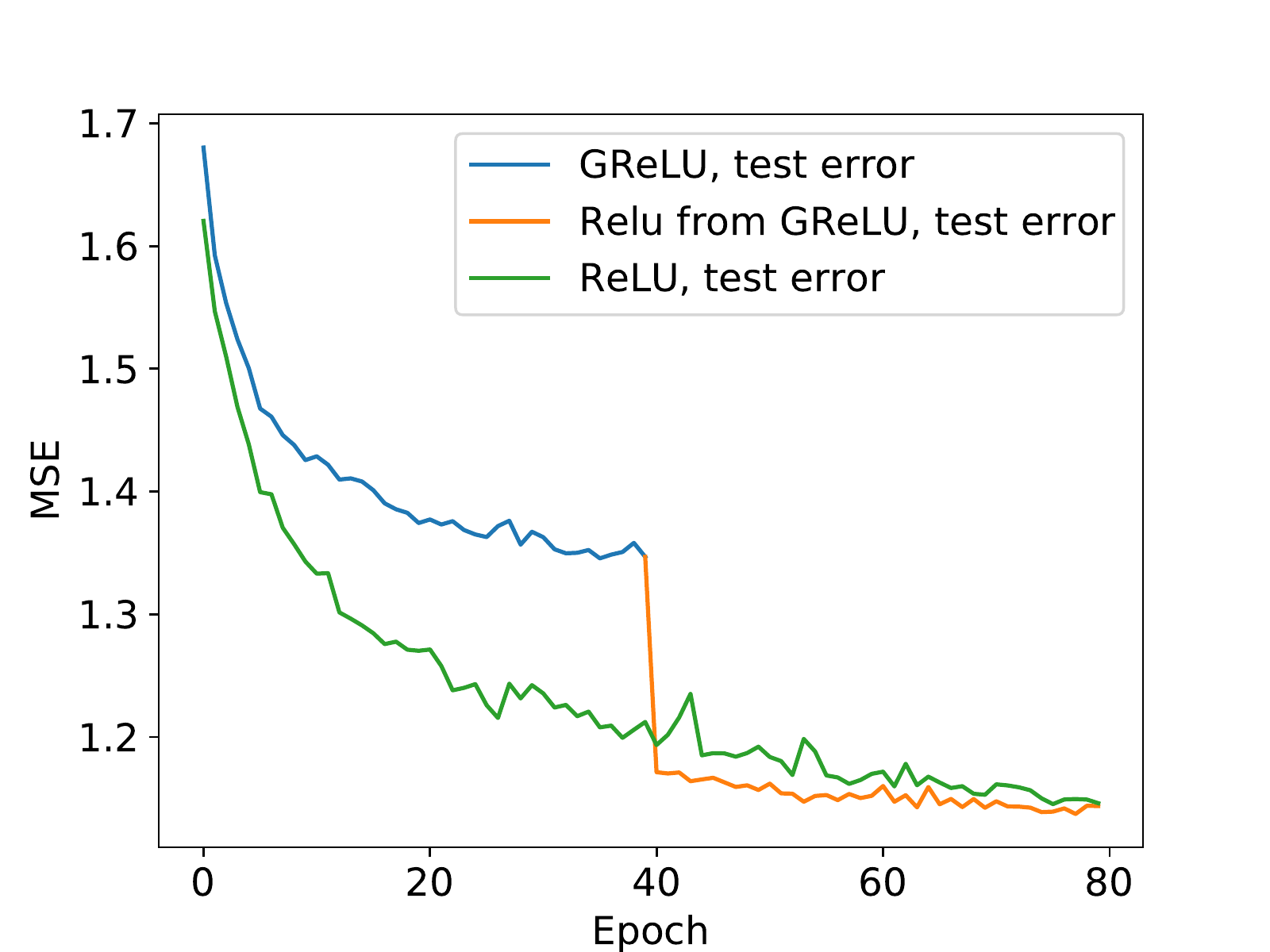}
    \caption{Test MSE vs epoch}
    \label{fig:test_loss_grelu_relu_cifar}
    \end{subfigure}
    \caption{MSE as a function of epochs for CIFAR-10 dataset. (a) Train with no augmentations. (b,c) Train and test losses for GReLU$\rightarrow$ReLU scheme with transform at epoch $40$.}
  \label{fig:grelu_relu_cifar}
\end{figure}
In this section, in order to answer \textbf{question 2}, we evaluate the training of ReLU and GReLU networks in more practical scenarios, by experimenting with sub-theoretic widths on two popular datasets from different domains, input dimensions and structures: \\
\textbf{CIFAR-10} the popular image classification dataset, with $n=50,000$ 
and $d=32\times 32 \times 3=3072$.
We  transform each classification to a scalar regression by calculating the squared loss between the one-hot $10$ dimensional vector encoding its correct label and the output logits of the network. \\
\begin{wraptable}{r}{0.35\textwidth}
  \small
\caption{CIFAR-$10$ MSE}
     \label{tab:experiment}
    \centering
    \begin{tabular}{|c|c|c|}
    \hline
         \textbf{Architecture} & \textbf{ReLU} & \textbf{GRelu} \\ \hline
         \makecell{Depth $6$, \\ Width $10$k} & $0.48$ & $0.02$  \\ \hline
         \makecell{Depth $7$, \\Width $10$k} & $0.27$ & $0.01$  \\ \hline
         \makecell{Depth $4$, \\Width $15$k} & $0.75$ & $0.03$  \\ \hline
    \end{tabular}
\end{wraptable}
\textbf{SGEMM}
the well known CPU kernel-performance regression dataset\footnote{\url{https://archive.ics.uci.edu/ml/datasets/SGEMM+GPU+kernel+performance}}~\citep{nugteren2015cltune}. It contains $n=241,600$ samples of dimension $d\!=\!18$, measuring CPU runtime of matrix products. 
SGEMM experiments demonstrate a similar behavior to CIFAR-10, and appear in Section \ref{sec:sgemm} of the supplementary material.

We initialize both networks with the same trainable weights $\bar{W}_0$ and additional $\Psi_{L}$ for GReLU,
and train them with a fixed learning rate of $10^{-4}$.
The results presented in Table~\ref{tab:experiment} show that all variants of GReLU converged to near zero losses. Conversely, ReLU variants improved with increased depth, but did not reach similar losses.
Possibly deeper ReLU networks with fine-tuned learning rates schemes would reach near zero losses as well. 
The loss progression
of different variants is plotted in Figure~\ref{fig:train_loss_cifar}. 
The GReLU networks (blue)
converge with higher rates as well.

We proceed to evaluate the \textbf{GReLU$\rightarrow$ReLU} scheme proposed in Section \ref{sec:relu_equivalence}.
 While this paper is focused on optimization, comparing the empirical \textit{generalization} of this scheme against traditional ReLU network is intriguing. For a fair and simple comparison with no additional hyper-parameters to tune, we train both networks with RandAugment~\citep{randaugment} to avoid overfitting. We train the variants for $80$ epochs, and transform the GReLU to its equivalent ReLU network exactly in the middle.
 Results on train and test set (additional $10,000$ examples) appear in Figures (\ref{fig:train_loss_grelu_relu_cifar},\ref{fig:test_loss_grelu_relu_cifar}).
 The train loss of the GReLU variant is continuous at epoch $40$, as guaranteed by Theorem~\ref{thm:equivalence2}, and comparable to the ReLU variant. However, on the test set a significant loss drop can be seen as a result of the transform, leading to an improved performance of the GReLU network. Since GReLU is transformed to a ReLU network with minimal $\ell_2$-norm, the transform effectively acts as a ridge-regularization step.  
 While the final loss is similar, we note that the GReLU training is almost twice as fast, as explained in Section~\ref{sec:implementation}. This can be used for further improvement.
\subsection{Training Dynamics}
To answer \textbf{question 3}, we start with a comparison of the optimization trajectories of the networks. Regular trajectories are favored due to typically smoother corresponding optimization landscapes. Acceleration methods like gradient-descent with momentum~\citep{qian1999momentum} provide more regular trajectories and are commonly used. During a full train with each network with $\eta=10^{-4}$, we compute the $\ell_2$ norm between consecutive gradients, $\|g_t-g_{t-1}\|_2$, and present the results in Figure~\ref{fig:train_loss_grad}. The gradient differences of the GReLU network are smaller than the ones of the ReLU network, and
gradually decrease along the convergence to the global optima.
This provides an additional empirical validation of the relatively larger learning rates allowed by our theory. In addition, this empirical evaluation matches the tighter bound on the gradients difference guaranteed by theorem~\ref{thm:NTK}. The gradient differences of the ReLU network increase monotonically along the training, with highly irregular trajectories that match the relatively high variation of the losses in Figure~\ref{fig:train_loss_sgemm3}.
\\
To better understand the behaviour of the ReLU network, we conduct an additional experiment, of calculating the change in its activation patterns along the training. More specifically, we present the mean Hamming distance between its activation patterns at consecutive epochs. The distance is computed across all ReLU activations over a selected validation set.
The results are illustrated in Figure~\ref{fig:successive_relu_hamming} and are quite surprising: a large portion of the activations, $20\%-25\%$, changes every epoch and keep increasing along the run. This result complements the previous one by showing some instability of the ReLU network. 
We point out that some level of change in activation patterns along the optimization possibly contributes to the generalization via implicit regularization.
In such case, a hybrid version of the GReLU network that allows gradual, controlled changes of its activation patterns is an interesting future direction to be pursued.

\section{Related Work} \label{sec:related}
\textbf{Neural network overparameterization} has been extensively studied under different assumptions over the training data and network architecture.
For the simplified case of well-separated data, \cite{ji2018risk} showed that a poly-logarithmic width is sufficient to guarantee convergence in two-layer networks, and \cite{chen2019much} extended the result to deep networks using the NTRF function class by~\cite{cao2019generalization}. For mixtures of well-separated distributions, ~\cite{li2018learning} proved that when training two-layer networks with width proportional to a high-order polynomial in the parameters of the problem with SGD, small train and generalization errors are achievable. Other common assumptions include Gaussian input~\citep{brutzkus2017globally,zhong2017learning,li2017convergence}, independet activations~\citep{choromanska2015loss,kawaguchi2016deep}, and output generated by a teacher-network~\citep{li2017convergence,zhang2019learning,brutzkus2017globally}. In our work, only a mild assumption of non-degenerate input is used.\\
overparameterization of special neural networks is varied across the analysis of linear networks \citep{kawaguchi2016deep,arora2018convergence,bartlett2018gradient}, smooth activations~\citep{du2018power,kawaguchi2019gradient}, and shallow networks~\citep{du2018power,oymak2020towards}. For deep linear networks, \cite{kawaguchi2016deep} proved that all local minima are global minima. \cite{arora2018convergence} proved  trainability for deep linear networks of minimal width, generalizing result by \cite{bartlett2018gradient} on residual linear networks. 
\cite{du2018power} considered neural networks with smooth activations, and showed that
in overparameterized shallow networks with quadratic activation, all local minima are global minima. 
\cite{kawaguchi2019gradient} showed that when only the output layer is trained, i.e. least squares estimate with gradient descent, tighter bounds on network width are achievable. 
Such training assumptions are rarely used in practice due to poor generalization. 
Considering one-hidden-layer networks, \cite{du2018gradient} showed that under the assumption of non-degenerate population Gram matrix, gradient-descent converges to a globally optimal solution.
\cite{oymak2020towards} stated the current tightest bound on the width of shallow networks for different activation functions. For the ReLU activation, the bound we derive for deep networks improves theirs over shallow networks as well.
\\
We focus on the general setting of deep networks equipped with ReLU activation, as typically used in practice. 
Similar works~\citep{zou2018stochastic,du2019gradient,allen2019convergence,zou2019improved} require unrealistic overparameterization conditions, namely a minimal width proportional to the size of the dataset in the power of $(8,24,26)$~\citep{zou2019improved,allen2019convergence,zou2018stochastic}, and the depth of the network in the power of $(12$-$38)$~~\citep{zou2019improved,zou2018stochastic} or exponential dependency~\citep{du2019gradient}.
Our work improves the current tightest bound~\citep{zou2019improved} by a factor of $O\paren{n^6 L^{11}}$, with a corresponding significant decrease in convergence time, from polynomial to logarithmic in the parameters of the problem.  
\\
\textbf{Optimization landscape of deep networks} under different assumptions is an active line of research.
A few works showed that ReLU networks have bad local minima~\citep{venturi2018spurious,safran2018spurious,swirszcz2016local,yun2017global} due to "dead regions", even for random input data and labels created according to a planted model~\citep{yun2017global,safran2018spurious},
a property which is not shared with smooth activations~\citep{liang2018understanding,du2018power,soltanolkotabi2018theoretical}. Specifically, \cite{hardt2016identity} showed
that deep linear residual networks have no spurious local minima.
\\
\textbf{Training dynamics of ReLU networks} and their activation patterns were studied by few recent works~\citep{hanin2019complexity,hanin2019deep,li2018learning}.
\cite{hanin2019deep} showed that the average number of patterns along the training is significantly smaller than the possible one and only weakly depends on the depth. We study a single random pattern, entirely decoupled from depth.
\cite{li2018learning} studied "pseudo gradients" for shallow networks, close in spirit with the fixed ReLU activations suggested by our work. They showed that unless the generalization error is small, pseudo gradients remain large, providing motivation for fixing activation patterns from a generalization perspective. Furthermore, they showed that pseudo gradients could be coupled with original gradients for most cases when the network is sufficiently overparameterized. \\
\textbf{Neural Tangent Kernel} regression has been shown to describe the evolution of infinitely wide fully-connected networks when trained with gradient descent~\citep{jacot2018neural}. Additional works provided corresponding kernels for residual~\citep{huang2020deep} and convolutional networks~\citep{li2019enhanced}, allowing to study the training dynamics of those in functional space. 
Few works~\citep{belkin2018understand,cao2019generalization,belkin2018overfitting,liang2020just} analyzed NTK for explaining properties for deep neural networks, most notably their trainability and generalizability. 
An important result by \cite{allen2019convergence} extended the NTK equivalence to deep networks of finite width. We provide an improved condition on the required width for this equivalence to hold.

\section{Conclusions and Future Work} \label{sec:conclusion}
In this paper we proposed a novel technique for studying overparameterized deep ReLU networks which leads to state-of-the-art guarantees on trainability. Both the required network size and convergence time significantly improve previous theory and can be tested in practice for the first time.  Further theoretical and empirical analysis provide insights on the optimization of neural networks with first-order methods.
Our analysis and proof technique can be extended to additional setups such as optimization with stochastic gradient-descent, cross-entropy loss, and other network architectures such as convolutional and residual networks similarly to~\citep{zhang2019training,allen2019convergence,du2018gradient}. 
Our encouraging empirical results motivate analyzing additional variants like Gated-CNNs and Gated-ResNets. 
Our improved finite-width NTK equivalence enables studying neural networks of practical sizes using kernel theory, potentially leading to a better understanding of learning dynamics and generalization. 
\section*{Acknowledgements}
We would like to thank Assaf Hallak, Tal Ridnik, Avi Ben-Cohen, Yushun Zhang and Itamar Friedman for their feedbacks and productive discussions. 
\bibliography{references}
\newpage
\begin{center} \begin{huge}{Appendix} \end{huge} \end{center}
\label{appendix}
\section{Additional experiments on SGEMM dataset} \label{sec:sgemm}
We start with training networks of width $4096$ and depth $10$ over SGEMM, and plot their losses along the training with varied learning rates in Figures \ref{fig:train_loss_cst} for the GReLU and \ref{fig:train_loss_relu} for the ReLU network. 
While for smaller rates results are quite similar, for larger ones the losses of the GReLU network decrease more monotonically and reach lower values. A comparison of high learning rates in Figure~\ref{fig:train_loss_sgemm3} shows that ReLU network oscillates with larger amplitudes towards higher loss values.
The losses of the GReLU network decrease more monotonically, and reach smaller values with larger learning rates. Both networks do not reach zero loss, as their widths are smaller than required by our theory.
We further test GReLU with a fixed learning rate for different widths and depths, and present the results in Figure~\ref{fig:train_loss_sizes}. 
While both increased depth and width improve the results, deeper and narrower networks perform better for the same number of parameters. This result is consistent with similar comparisons of ReLU networks in previous works.

\begin{figure}[h]
    \centering
     \begin{subfigure}[b]{0.35\textwidth}
         \centering
            \includegraphics[width=\textwidth]{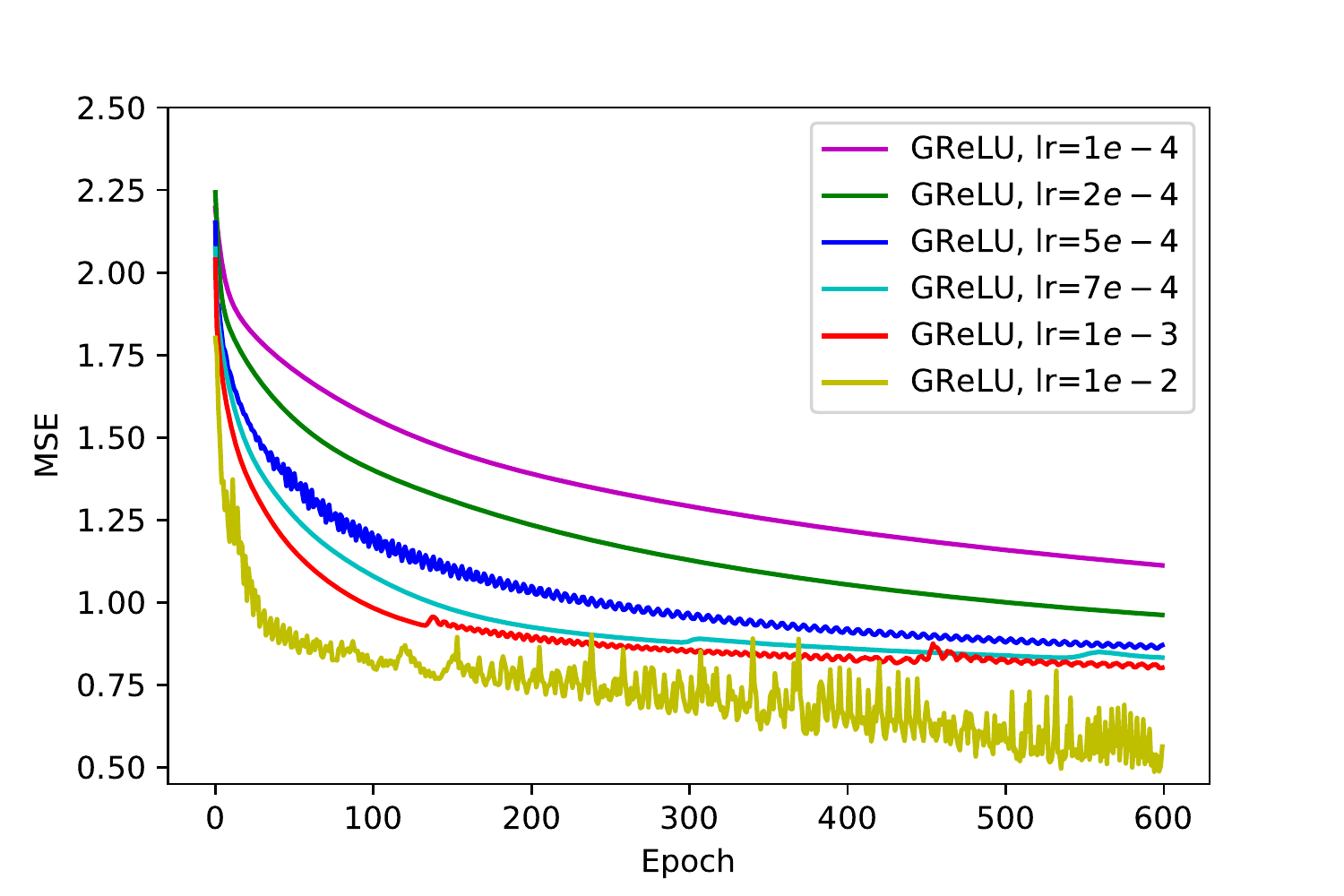}
            \caption{GReLU}
            \label{fig:train_loss_cst}
    \end{subfigure} \hspace{-0.8cm}
    \begin{subfigure}[b]{0.35\textwidth}
         \centering
            \centering
    \includegraphics[width=\textwidth]{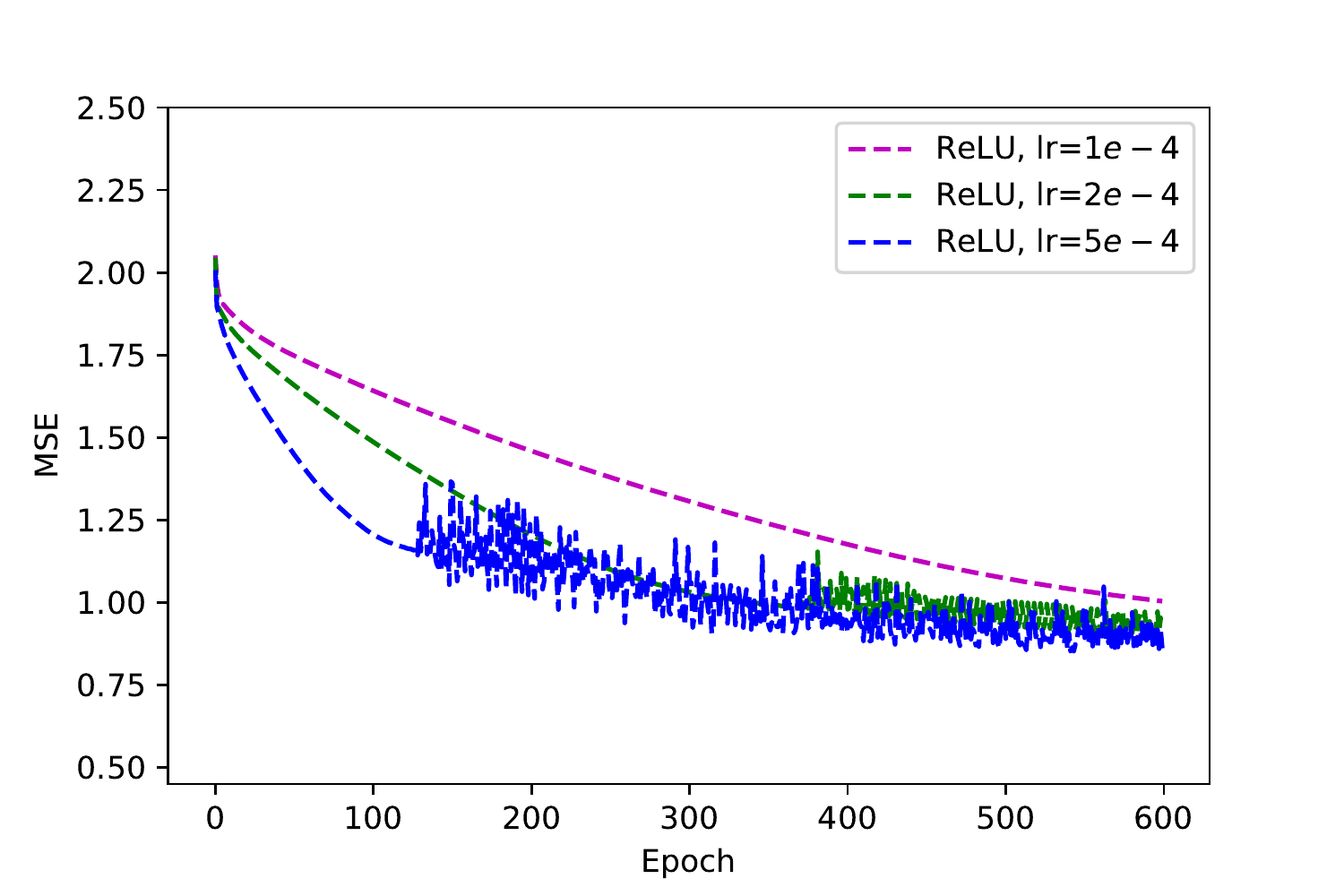}
    \caption{ReLU}
    \label{fig:train_loss_relu}
    \end{subfigure} \hspace{-0.8cm}
    \begin{subfigure}[b]{0.35\textwidth}
         \centering
            \centering
    \includegraphics[width=\textwidth]{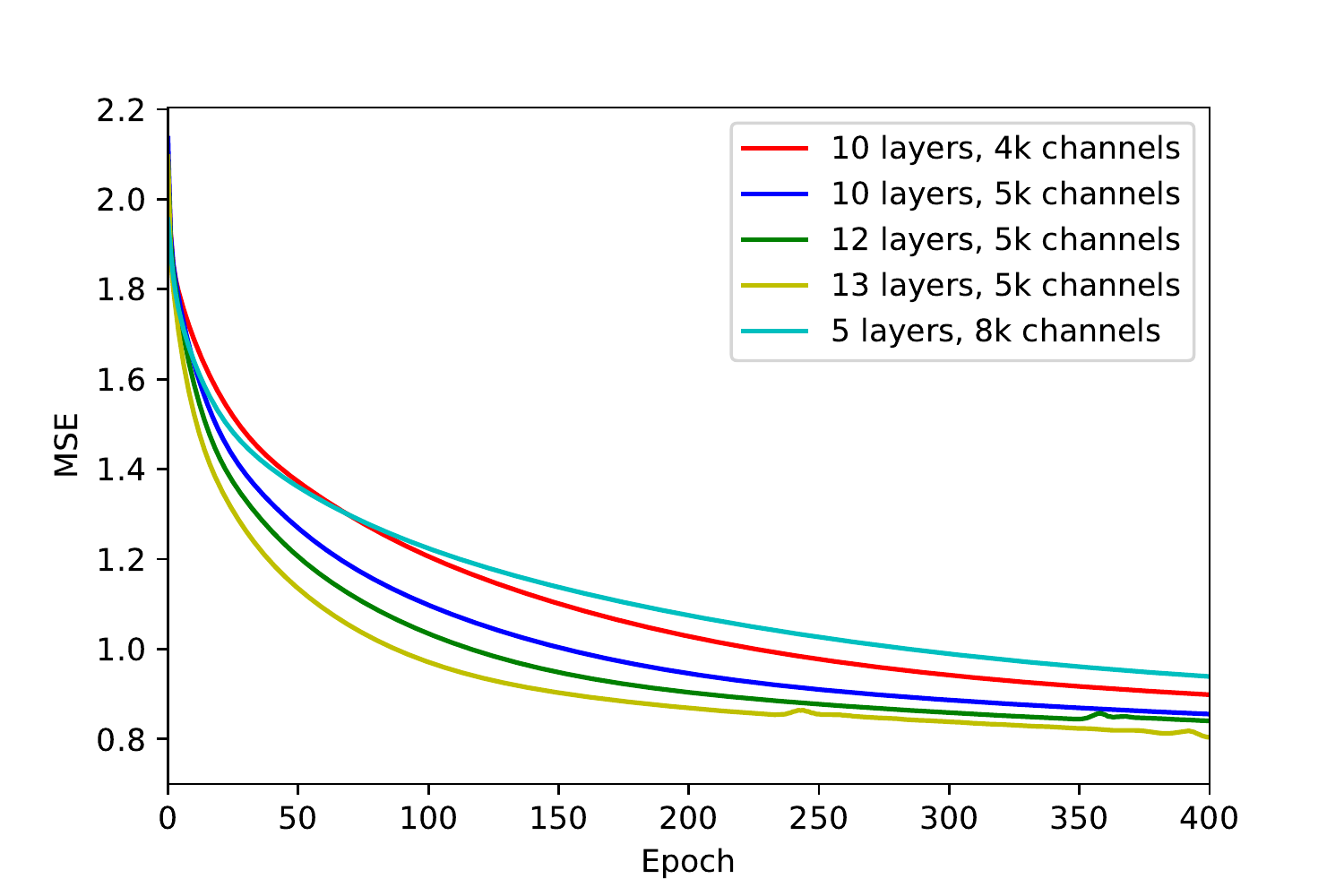}
    \caption{GReLU}
    \label{fig:train_loss_sizes}
    \end{subfigure}
    \caption{MSE as a function of epochs for SGEMM dataset}
\end{figure}

\begin{figure}[h]
     \centering
     \begin{subfigure}[b]{0.35\textwidth}
         \centering
            \includegraphics[width=\textwidth]{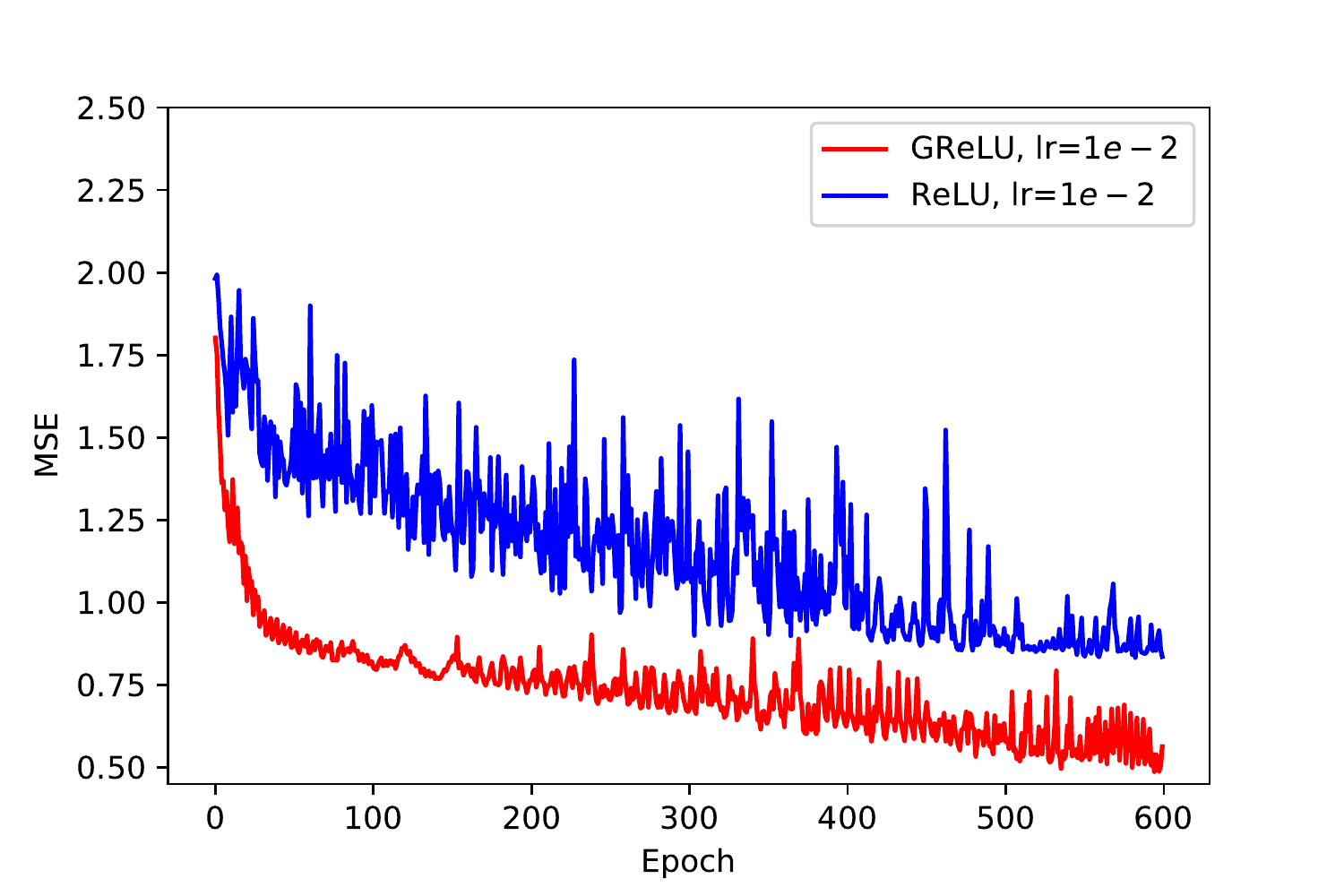}
        \caption{$\eta=10^{-2}$}
    \end{subfigure} \hspace{-0.8cm}
     \begin{subfigure}[b]{0.35\textwidth}
         \centering
            \includegraphics[width=\textwidth]{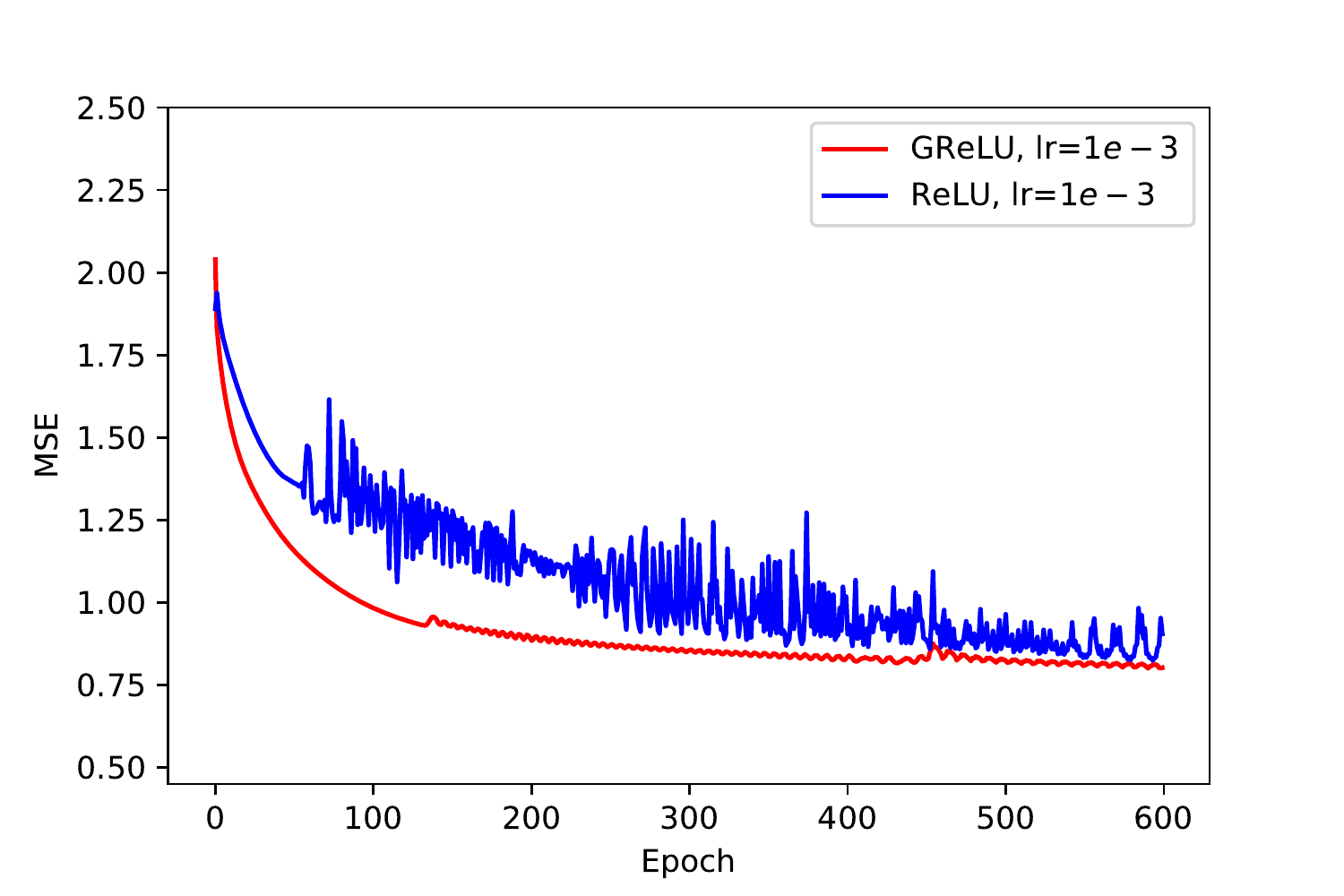}
        \caption{$\eta=10^{-3}$}
    \end{subfigure} \hspace{-0.8cm}
         \begin{subfigure}[b]{0.35\textwidth}
         \centering
            \includegraphics[width=\textwidth]{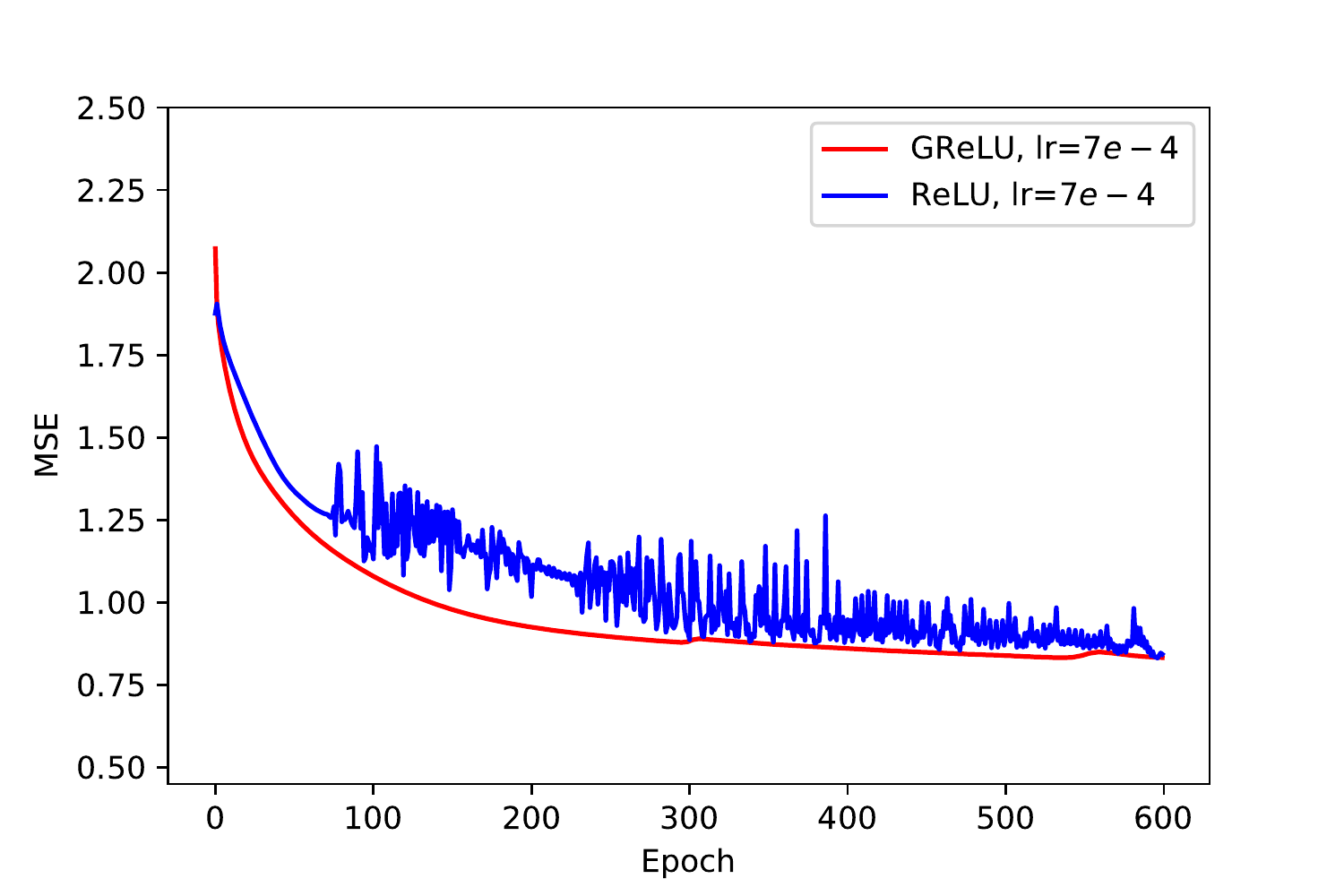}
        \caption{$\eta=7\cdot 10^{-4}$}
    \end{subfigure}
    \caption{MSE vs epochs per learning rate for GReLU and ReLU over SGEMM dataset} 
    \label{fig:train_loss_sgemm3}
\end{figure}

\section{Proofs for Section~\ref{sec:proof_sketch}} \label{sec:proofs_for_sketch}
\subsection{Proof of Lemma~\ref{lemma:diff:W}}
\begin{proof} 
By the definition of $W_t^i$ in (\ref{notation:parameter}) we have
\begin{align*}
W^i_{t+1} - W^i_t 
= & BD_L^iW_{t+1, L}\ldots D_1^i W_{t+1, 1}D_0^iC - BD_L^iW_{t, L}\ldots D_1^i W_{t, 1}D_0^iC \\
\overset{(a)}{=} & BD_L^i\left(W_{t,L} - \eta\nabla_L\ell(W_t)\right)\ldots D_1^i\left(W_{t,1} - \eta\nabla_1\ell(W_t)\right)D_0^iC - BD_L^iW_{t, L}\ldots D_1^i W_{t, 1}D_0^iC \\
\overset{(b)}{=} & \eta^2\Delta_{t,i} - \eta \underbrace{\sum_{k=1}^L BD_L^iW_{t,L}\ldots D_{k}^i\nabla_k\ell(W_t)D_{k-1}^iW_{t,k-1}\ldots D_1^iW_{t,1}D_0^iC,}_{:=Z_t^i}
\end{align*}
where (a) is due to the update of $W_{t+1,k} = W_{t,k} - \eta \nabla_k\ell(W_t)$ for any $k\in[L]$; (b) uses the definition of $\Delta_{t,i}$ in (\ref{eqn:delta}). In the above, we divide $W^i_{t+1} - W^i_t$ into two parts, with $\eta Z_t^i$ is proportional to $\eta$ and $\eta^2\Delta_{t,i}$ is proportional to $\eta^2$. We can simplify $Z_t^i$ as
\begin{align*}
Z_t^i 
\overset{(\ref{notation:gradient:parameter})}{=}& \sum_{k=1}^L F_{t,k+1}^i\nabla_k\ell(W_t)G_{t,k-1}^i \\
\overset{(\ref{eqn:grad})(\ref{eqn:gamma})}{=}& \sum_{k=1}^L F_{t,k+1}^i\left[F_{t,k+1}^i\right]^{\top}\left(W_t^i - \Phi_i\right)x_ix_i^{\top}\left[G_{t,k-1}^i\right]^{\top}G_{t,k-1}^i+\Gamma_{t,i}. 
\end{align*}
We complete the proof by plugging in the expression of $Z_t^i$ into the expression of $W_{t+1}^i - W_t^i$.
\end{proof}

\subsection{Proof of Lemma~\ref{lemm:descent}}
\begin{proof}
By the definition of loss function in (\ref{notation:loss}), we have
\begin{align}\label{eqn:temp-1}
    \nonumber\ell(W_{t+1}) - \ell(W_{t}) 
    \nonumber=& \sum_{i=1}^n \left\{ \frac{1}{2}\nm{(W_{t+1}^i - \Phi_i)x_i}^2 - \frac{1}{2}\nm{(W_t^i - \Phi_i)x_i}^2\right\}\\
    =& \sum_{i=1}^n \left\{ \langle (W_{t+1}^i - W_{t}^i)x_i,  (W_{t}^i-\Phi_i) x_i\rangle + \frac{1}{2}\nm{(W_t^i - W_{t+1}^i)x_i}^2\right\}.
\end{align}
For the first term in (\ref{eqn:temp-1}), by Lemma~\ref{lemma:diff:W}, we have
\begin{align}\label{eqn:temp-2}
\nonumber& \langle (W_{t+1}^i - W_t^i)x_i, (W_t^i - \Phi_i)x_i\rangle \\
\nonumber= & -\eta \sum_{k=1}^L \left\langle F_{k+1}^i\left[F_{k+1}^i\right]^{\top}\left(W_t^i - \Phi_i\right)x_ix_i^{\top}\left[G_{t,k-1}^i\right]^{\top}G_{t,k-1}^ix_i, (W_t^i - \Phi_i)x_i\right\rangle + \eta^2\langle \Delta_{t,i}x_i, (W_t^i - \Phi_i)x_i\rangle\\
\nonumber & -\eta\sum_{k=1}^L \sum_{j\neq i}\left\langle F_{t,k+1}^i\left[F_{t,k+1}^j\right]^{\top}\left(W_t^j - \Phi_j\right)x_jx_j^{\top}\left[G_{t,k-1}^j\right]^{\top}G_{t,k-1}^ix_i, (W_t^i - \Phi_i)x_i\right\rangle\\
\nonumber= & -\eta \sum_{k=1}^L \left\langle F_{k+1}^i\left[F_{k+1}^i\right]^{\top}\left(W_t^i - \Phi_i\right)x_ix_i^{\top}\left[G_{t,k-1}^i\right]^{\top}G_{t,k-1}^ix_i, (W_t^i - \Phi_i)x_i\right\rangle + \eta^2\langle \Delta_{t,i}x_i, (W_t^i - \Phi_i)x_i\rangle\\
& -\eta\sum_{k=1}^L \sum_{j\neq i} \left\langle G_{t,k-1}^jx_j, G_{t,k-1}^ix_i\right\rangle \left\langle (W_t^j - \Phi_j)x_j, F_{t,k+1}^j \left[F_{t,k+1}^i\right]^{\top}(W_t^i - \Phi_i)x_i\right\rangle,
\end{align}
where the last equality is due to the facts that $x_j^{\top}\left[G_{t,k-1}^j\right]^{\top}G_{t,k-1}^ix_i$ is a scalar and $\langle Ax, y\rangle = \langle x, A^\top y\rangle$ for any matrix $A$ and vectors $x,y$. On the other hand, the first term in (\ref{eqn:temp-2}) will be
\begin{align}\label{eqn:temp-3}
   \nonumber & \left\langle F_{k+1}^i\left[F_{k+1}^i\right]^{\top}\left(W_t^i - \Phi_i\right)x_ix_i^{\top}\left[G_{t,k-1}^i\right]^{\top}G_{t,k-1}^ix_i, (W_t^i - \Phi_i)x_i\right\rangle\\
    \nonumber =&\text{Tr}\left( x_i^\top(W_t^i - \Phi_i)^\top F_{k+1}^i\left[F_{k+1}^i\right]^{\top}\left(W_t^i - \Phi_i\right)x_ix_i^{\top}\left[G_{t,k-1}^i\right]^{\top}G_{t,k-1}^ix_i \right)\\
   \nonumber  \overset{(a)}{=}&\text{Tr}\left(\left[G_{t,k-1}^i\right]^{\top}G_{t,k-1}^ix_i x_i^\top(W_t^i - \Phi_i)^\top F_{k+1}^i\left[F_{k+1}^i\right]^{\top}\left(W_t^i - \Phi_i\right)x_ix_i^{\top}\right)\\
    \nonumber  \overset{(b)}{=}& \mbox{vec}^{\top}(\left(W_t^i - \Phi_i\right)x_ix_i^{\top}) \mbox{vec}\left(\left(\left[G_{t,k-1}^i\right]^{\top}G_{t,k-1}^ix_i x_i^\top(W_t^i - \Phi_i)^\top F_{k+1}^i\left[F_{k+1}^i\right]^{\top}\right)^\top\right)\\
    \nonumber  \overset{(c)}{=}& \mbox{vec}^{\top}(\left(W_t^i - \Phi_i\right)x_ix_i^{\top}) \mbox{vec}\left(F_{k+1}^i\left[F_{k+1}^i\right]^{\top}(W_t^i - \Phi_i)x_i x_i^\top \left[G_{t,k-1}^i\right]^{\top}G_{t,k-1}^i \right)\\
     \overset{(d)}{=}& \mbox{vec}^{\top}(\left(W_t^i - \Phi_i\right)x_ix_i^{\top}) \left(\left[G_{t,k-1}^i\right]^{\top}G_{t,k-1}\right)\otimes \left(F_{k+1}^i\left[F_{k+1}^i\right]^{\top}\right) \mbox{vec}\left((W_t^i - \Phi_i)x_i x_i^\top  \right),
\end{align}
where for any matrices $A,B,X$: (a) uses $\text{Tr}(AB) = \text{Tr}(BA)$; (b) uses $\text{Tr}(AB) = \mbox{vec}^{\top}(B) \mbox{vec}(A^\top)$; (c) uses $(AB)^\top = B^\top A^\top$; (d) uses $\mbox{vec}(AXB) = B^\top\otimes A \mbox{vec}(X)$.
Define
\begin{align}
    A_{t,k}^i :=& \left(\left[G_{t,k-1}^i\right]^{\top}G_{t,k-1}^i\right)\otimes\left(F_{t,k+1}^i\left[F_{t,k+1}^i\right]^{\top}\right),\label{eqn:def:A}\\
    \Xi_{t,k}^{i,j} :=& \langle G^j_{t,k-1}x_j, G^i_{t,k-1}x_i\rangle F_{t,k+1}^j\left[F_{t,k+1}^i\right]^{\top}, \label{eqn:def:xi}
\end{align}
where $\otimes$ is the Kronecker product. By plugging (\ref{eqn:temp-3}) into (\ref{eqn:temp-2}) and using the definitions in (\ref{eqn:def:A}) and (\ref{eqn:def:xi}), we have
\begin{align}\label{eqn:temp-4}
\nonumber& \sum_{i=1}^n \langle (W_{t+1}^i - W_t^i)x_i, (W_t^i - \Phi_i)x_i\rangle \\
\nonumber= & -\eta \sum_{i=1}^n\mbox{vec}^{\top}\left((W_t^i - \Phi_i)x_ix_i^{\top}\right)\left(\sum_{k=1}^L A_{t,k}^i\right)\mbox{vec}\left((W_t^i - \Phi_i)x_ix_i^{\top}\right) + \eta^2\sum_{i=1}^n\langle \Delta_{t,i}x_i, (W_t^i - \Phi_i)x_i\rangle\\
&- \eta\sum_{i=1}^n\sum_{k=1}^L\sum_{j\neq i} \left\langle (W_t^j - \Phi_j)x_j, \Xi_{t,k}^{i,j}(W_t^i - \Phi_i)x_i\right\rangle
\end{align}
By using the fact that $(A\otimes B)^\top = A^\top\otimes B^\top$, we know that $(A_{t,k}^i)^\top = A_{t,k}^i$, which means that $A_{t,k}^i$ is symmetric. Thus we have
\begin{align}\label{eqn:temp-5}
\mbox{vec}^{\top}\left((W_t^i - \Phi_i)x_ix_i^{\top}\right)\left(\sum_{k=1}^L A_{t,k}^i\right)\mbox{vec}\left((W_t^i - \Phi_i)x_ix_i^{\top}\right) \ge \nm{(W_t^i - \Phi_i)x_i}^2 \left(\sum_{k=1}^L \lambda_{\min}\left(A_{t,k}^i\right)\right).
\end{align}
For the last term in (\ref{eqn:temp-4}),
\begin{align}\label{eqn:temp-6}
\nonumber& - \eta\sum_{i=1}^n\sum_{k=1}^L\sum_{j\neq i} \left\langle (W_t^j - \Phi_j)x_j, \Xi_{t,k}^{i,j}(W_t^i - \Phi_i)x_i\right\rangle \\
\nonumber\overset{(a)}{\le} & \frac{\eta}{2}\sum_{i=1}^n\sum_{k=1}^L\sum_{j\neq i} \nm{\Xi_{t,k}^{i,j}}_2\left(\nm{(W_t^i - \Phi_i)x_i}^2 + \nm{(W_t^j - \Phi_j)x_j}^2\right)\\
\nonumber\overset{\text{Lemma}~\ref{tech:lemm:1}}{=} &  \eta\sum_{i=1}^n \sum_{k=1}^L\nm{(W_t^i - \Phi_i)x_i}^2\left(\sum_{j \neq i} \frac{\left\|\Xi_{t,k}^{i,j}\right\|_2+\left\|\Xi_{t,k}^{j,i}\right\|_2}{2}\right)\\
\overset{(b)}{=} &  \eta\sum_{i=1}^n \sum_{k=1}^L\nm{(W_t^i - \Phi_i)x_i}^2\left(\sum_{j \neq i} \left\|\Xi_{t,k}^{i,j}\right\|_2\right)
\end{align}
where (a) uses the fact that $\langle x, A y \rangle \le \|x\|\|Ay\| \le \|x\| \|A\|_2 \|y\| \le \frac{1}{2}\|A\|_2 (\|x\|^2 + \|y\|^2)$ for any matrix $A\in\R^{d\times d}$ and vectors $x, y \in\R^d$; (b) uses the fact of definition of $\Xi_{t,k}^{i,j}$ that $\|\Xi_{t,k}^{i,j}\|_2 = \|\Xi_{t,k}^{j,i}\|_2$. 
As a result, by (\ref{eqn:temp-4}) (\ref{eqn:temp-5}) (\ref{eqn:temp-6}) we have
\begin{align}\label{eqn:temp-7}
\nonumber&\sum_{i=1}^n\langle (W_{t+1}^i - W_t^i)x_i, (W_t^i - \Phi_i)x_i\rangle\\
\nonumber\leq & - \eta\sum_{i=1}^n \nm{(W_t^i - \Phi_i)x_i}^2 \left(\sum_{k=1}^L \lambda_{\min}\left(A_{t,k}^i\right) \right) + \eta^2\sum_{i=1}^n \langle \Delta_{t,i}, (W_t^i - \Phi_i)x_ix_i^{\top}\rangle \\
&+ \eta\sum_{i=1}^n \sum_{k=1}^L\nm{(W_t^i - \Phi_i)x_i}^2\left(\sum_{j \neq i} \left\|\Xi_{t,k}^{i,j}\right\|_2\right).
\end{align}
We now handle the second term in (\ref{eqn:temp-1}), by Lemma~\ref{lemma:diff:W} and using Cauchy-Schwarz inequality, we know 
\begin{align}\label{eqn:temp-8}
\nonumber& \frac{1}{2}\nm{(W_{t+1}^i - W_t^i)x_i}^2\\
\nonumber\leq & \frac{3\eta^2}{2}\nm{\sum_{k=1}^L F_{t,k+1}^i\left[F_{t,k+1}^i\right]^{\top}(W_t^i - \Phi_i)x_i \nm{G_{t,k-1}^i x_i}^2 }^2 + \frac{3\eta^4}{2}\nm{\Delta_{t,i}x_i}^2 \\
\nonumber& + \frac{3\eta^2}{2}\nm{\sum_{k=1}^L\sum_{j \neq i} F_{t,k+1}^i \left[F_{t,k+1}^j\right]^{\top}(W_t^j - \Phi_j)x_j\langle G_{t,k-1}^jx_j, G_{t,k-1}^ix_i\rangle}^2 \\
\nonumber\leq & \frac{3\eta^2 L}{2}\sum_{k=1}^L\nm{F_{t,k+1}^i\left[F_{t,k+1}^i\right]^{\top}(W_t^i - \Phi_i)x_i}^2 \nm{G_{t,k-1}^i x_i}^4 + \frac{3\eta^4}{2}\nm{\Delta_{t,i}x_i}^2 \\
& + \frac{3\eta^2 L n}{2}\sum_{k=1}^L\sum_{j \neq i} \nm{F_{t,k+1}^i \left[F_{t,k+1}^j\right]^{\top}(W_t^j - \Phi_j)x_j}^2\left|\langle G_{t,k-1}^jx_j, G_{t,k-1}^ix_i\rangle\right|^2.
\end{align}
Since by using $\|Ax\| \le \|A\|_2\|x\|$ for any matrix $A\in\R^{d\times d}$ and vector $x \in\R^d$, we have
\begin{align}\label{eqn:temp-9}
    \nonumber\nm{F_{t,k+1}^i\left[F_{t,k+1}^i\right]^{\top}(W_t^i - \Phi_i)x_i}^2\le& \nm{F_{t,k+1}^i\left[F_{t,k+1}^i\right]^{\top}}_2^2 \nm{(W_t^i - \Phi_i)x_i}^2 \\
    \le& \lambda^2_{\max}\left(F_{t,k+1}^i\left[F_{t,k+1}^i\right]^{\top}\right)\nm{(W_t^i - \Phi_i)x_i}^2,
\end{align}
and
\begin{align}\label{eqn:temp-10}
    \nonumber&\sum_{k=1}^L\sum_{j \neq i} \nm{F_{t,k+1}^i \left[F_{t,k+1}^j\right]^{\top}(W_t^j - \Phi_j)x_j}^2\left|\langle G_{t,k-1}^jx_j, G_{t,k-1}^ix_i\rangle\right|^2\\
    \nonumber\le& \sum_{k=1}^L\sum_{j \neq i}\nm{(W_t^j - \Phi_j)x_j}^2 \nm{F_{t,k+1}^i \left[F_{t,k+1}^j\right]^{\top}}_2^2\left|\langle G_{t,k-1}^jx_j, G_{t,k-1}^ix_i\rangle\right|^2\\
    =&\sum_{k=1}^L\sum_{j \neq i}\nm{(W_t^j - \Phi_j)x_j}^2 \nm{\Xi_{t,k}^{i,j}}_2^2.
\end{align}
Hence, plugging (\ref{eqn:temp-9}) and (\ref{eqn:temp-10}) into (\ref{eqn:temp-8}), we have
\begin{align}\label{eqn:temp-11}
\nonumber&\frac{1}{2}\sum_{i=1}^n \nm{(W_{t+1}^i - W_t^i)x_i}^2 \\
\nonumber\leq & \frac{3\eta^2 L}{2}\sum_{i=1}^n \nm{(W_t^i - \Phi_i)x_i}^2 \left( \sum_{k=1}^L\lambda^2_{\max}\left(F_{t,k+1}^i\left[F_{t,k+1}^i\right]^{\top}\right)\nm{G_{t,k-1}^i x_i}^4\right) + \frac{3\eta^4}{2}\nm{\Delta_{t,i}x_i}^2 \\
\nonumber& + \frac{3\eta^2 L n}{2}\sum_{i=1}^n \sum_{k=1}^L\sum_{j \neq i}\nm{(W_t^j - \Phi_j)x_j}^2 \nm{\Xi_{t,k}^{i,j}}_2^2\\
\nonumber\leq & \frac{3\eta^2 L}{2}\sum_{i=1}^n \nm{(W_t^i - \Phi_i)x_i}^2 \left( \sum_{k=1}^L\lambda^2_{\max}\left(F_{t,k+1}^i\left[F_{t,k+1}^i\right]^{\top}\right)\nm{G_{t,k-1}^i x_i}^4\right) + \frac{3\eta^4}{2}\nm{\Delta_{t,i}x_i}^2 \\
& + \frac{3\eta^2 L n}{2}\sum_{i=1}^n \nm{(W_t^i - \Phi_i)x_i}^2 \sum_{k=1}^L\sum_{j \neq i} \nm{\Xi_{t,k}^{i,j}}_2^2,
\end{align}
where the last inequality uses
\begin{align*}
    \sum_{i=1}^n \sum_{k=1}^L\sum_{j \neq i}\nm{(W_t^j - \Phi_j)x_j}^2 \nm{\Xi_{t,k}^{i,j}}_2^2
    =& \sum_{k=1}^L \sum_{i=1}^n \sum_{j \neq i}\nm{(W_t^j - \Phi_j)x_j}^2 \nm{\Xi_{t,k}^{i,j}}_2^2\\
    \nonumber\overset{\text{Lemma}~\ref{tech:lemm:1}}{=}& \sum_{k=1}^L \sum_{i=1}^n \nm{(W_t^i - \Phi_i)x_i}^2 \sum_{j \neq i} \nm{\Xi_{t,k}^{j,i}}_2^2\\
   = &  \sum_{i=1}^n \nm{(W_t^i - \Phi_i)x_i}^2 \sum_{k=1}^L \sum_{j \neq i} \nm{\Xi_{t,k}^{j,i}}_2^2\\
   = &  \sum_{i=1}^n \nm{(W_t^i - \Phi_i)x_i}^2 \sum_{k=1}^L \sum_{j \neq i} \nm{\Xi_{t,k}^{i,j}}_2^2,
\end{align*}
where the last equality is due to $\nm{\Xi_{t,k}^{j,i}}_2 = \nm{\Xi_{t,k}^{i,j}}_2$. We complete the proof by using Young's inequality with $a_i>0$ to upper bound
\begin{align}\label{eqn:temp-12}
\langle \Delta_{t,i}, (W_t^i - \Phi_i)x_i\rangle \leq \frac{1}{2a_i}\nm{\Delta_{t,i}x_i}^2 + \frac{a_i}{2}\nm{(W_t^i - \Phi_i)x_i}^2
\end{align}
and using the property of the eigenvalue of the Kronecker product of two symmetric matrices
\begin{align}\label{eqn:temp-13}
\lambda_{\min}\left(A_{t,k}^i\right) = \lambda_{\min}\left(F_{t,k+1}^i\left[F_{t,k+1}^i\right]^\top\right)\lambda_{\min}\left(\left[G_{t,k-1}^i\right]^\top G_{t,k-1}^i\right).
\end{align}
\end{proof}

\subsection{Proof for Lemma~\ref{lemma:linear_rate}}
\begin{proof} \label{proof:lemma_linear_rate}
Using the assumption in (\ref{eqn:eig-lower-bound}), (\ref{eqn:eig-upper-bound}), (\ref{eqn:cross-upper-bound}) and (\ref{eqn:gamma-upper-bound}), by setting $a_i = \beta^4L^2$, 
\begin{align*}
\Lambda_i + \eta^2 a_i \le & -2\eta L \alpha^2  + 2\eta  L(n-1)\gamma \beta^2 + 3\eta^2L^2\beta^2\beta_x +  3\eta^2L^2n(n-1)\gamma^2\beta^4 + \eta^2L^2\beta^4\\
\le & -2\eta L \alpha^2  - 2\eta  L\gamma \beta^2 +  \eta  L \alpha^2  + 3\eta^2L^2\beta^2\beta_x - 3\eta^2L^2n\gamma^2\beta^4 +  \frac{3}{4}\eta^2L^2\alpha^2  + \eta^2L^2\beta^4\\
\le & -\eta L \alpha^2    + 3\eta^2L^2\beta^2\beta_x  +  \frac{3}{4}\eta^2L^2\alpha^2  + \eta^2L^2\beta^4.
\end{align*}
By choosing step size $\eta$ as 
$\eta \le \min\left(\frac{ \alpha^2}{12\beta^2\beta_xL}, \frac{1}{3 L}, \frac{\alpha^2}{ 4\beta^4 L}, \frac{1}{ \beta^2 L} \right)$, we have $ \Lambda_i \leq -\frac{3\eta\alpha^2 L}{4}$, and therefore
\begin{align}\label{eqn:lemm:descent:2:key:1}
    \ell(W_{t+1}) - \ell(W_t) \leq -\frac{3\eta\alpha^2 L}{4}\ell(W_t) + \frac{2\eta^2}{\beta^4L^2}\sum_{i=1}^n \nm{\Delta_{t,i}x_i}^2.
\end{align}
Next, we need to bound $\nm{\Delta_{t,i}x_i}^2$. Since $\nm{\Delta_{t,i}x_i}\le \nm{\Delta_{t,i}}_2 \nm{x_i} = \nm{\Delta_{t,i}}_2$, we want to bound $\nm{\Delta_t^i}_2$. To this end, we first bound $\nm{\nabla_k\ell(W_t)}_2^2$, i.e.
\begin{align}\label{eqn:lemm:descent:2:key:2}
\nonumber&\nm{\nabla_k\ell(W_t)}_2^2 \\
\nonumber\overset{(a)}{=}& \nm{\nabla_k\ell(W_t)^\top\nabla_k\ell(W_t)}_2\\
\nonumber\overset{(\ref{eqn:grad})}{=}&\left\|\left(\sum_{i=1}^n \left[F^i_{t,k+1}\right]^{\top}(W_t^i - \Phi_i)x_ix_i^{\top} \left[G^i_{t,k-1}\right]^{\top}\right)^\top\left(\sum_{i=1}^n \left[F^i_{t,k+1}\right]^{\top}(W_t^i - \Phi_i)x_ix_i^{\top} \left[G^i_{t,k-1}\right]^{\top}\right)\right\|_2\\
\nonumber\overset{(b)}{=}&\left\|\left(\sum_{i=1}^n G^i_{t,k-1} x_ix_i^\top(W_t^i - \Phi_i)^{\top} F^i_{t,k+1}\right)\left(\sum_{i=1}^n \left[F^i_{t,k+1}\right]^{\top}(W_t^i - \Phi_i)x_ix_i^{\top} \left[G^i_{t,k-1}\right]^{\top}\right)\right\|_2\\
\nonumber\overset{(c)}{\le}&\sum_{i=1}^n  \left\|G^i_{t,k-1} x_ix_i^\top(W_t^i - \Phi_i)^{\top} F^i_{t,k+1} \left[F^i_{t,k+1}\right]^{\top}(W_t^i - \Phi_i)x_ix_i^{\top} \left[G^i_{t,k-1}\right]^{\top}\right\|_2\\
& + \sum_{i=1}^n \sum_{j\ne i}\left\| G^i_{t,k-1} x_ix_i^\top(W_t^i - \Phi_i)^{\top} F^i_{t,k+1} \left[F^j_{t,k+1}\right]^{\top}(W_t^j - \Phi_j)x_j x_j^{\top} \left[G^j_{t,k-1}\right]^{\top}\right\|_2,
\end{align}
where (a) is due to $\|A\|_2^2 = \|A^\top A\|_2$ for matrix $A\in\R^{m\times m}$; (b) uses the facts that $(A+B)^\top=A^\top + B^\top$ and $(AB)^\top = B^\top A^\top$ for matrices $A, B\in\R^{m\times m}$; (c) uses $\|A+B\|_2 \le \|A\|_2 + \|B\|_2$ and $\|A\|_2 = \|A^\top\|_2$.
Since $\|A^\top A\|_2 = \|A\|_2^2\le\|A\|_F^2 = \text{tr}(A^\top A)$, then
\begin{align}\label{eqn:lemm:descent:2:key:3}
    \nonumber&\left\|G^i_{t,k-1} x_ix_i^\top(W_t^i - \Phi_i)^{\top} F^i_{t,k+1} \left[F^i_{t,k+1}\right]^{\top}(W_t^i - \Phi_i)x_ix_i^{\top} \left[G^i_{t,k-1}\right]^{\top}\right\|_2\\
    \nonumber\le& \text{Tr}\left( G^i_{t,k-1} x_ix_i^\top(W_t^i - \Phi_i)^{\top} F^i_{t,k+1} \left[F^i_{t,k+1}\right]^{\top}(W_t^i - \Phi_i)x_ix_i^{\top} \left[G^i_{t,k-1}\right]^{\top} \right)\\
   \nonumber  \overset{(a)}{=}&\text{Tr}\left(\left[G_{t,k-1}^i\right]^{\top}G_{t,k-1}^ix_i x_i^\top(W_t^i - \Phi_i)^\top F_{k+1}^i\left[F_{k+1}^i\right]^{\top}\left(W_t^i - \Phi_i\right)x_ix_i^{\top}\right)\\
    \nonumber  \overset{(b)}{=}& \mbox{vec}^{\top}(\left(W_t^i - \Phi_i\right)x_ix_i^{\top}) \mbox{vec}\left(\left(\left[G_{t,k-1}^i\right]^{\top}G_{t,k-1}^ix_i x_i^\top(W_t^i - \Phi_i)^\top F_{k+1}^i\left[F_{k+1}^i\right]^{\top}\right)^\top\right)\\
    \nonumber  \overset{(c)}{=}& \mbox{vec}^{\top}(\left(W_t^i - \Phi_i\right)x_ix_i^{\top}) \mbox{vec}\left(F_{k+1}^i\left[F_{k+1}^i\right]^{\top}(W_t^i - \Phi_i)x_i x_i^\top \left[G_{t,k-1}^i\right]^{\top}G_{t,k-1}^i \right)\\
    \nonumber \overset{(d)}{=}& \mbox{vec}^{\top}(\left(W_t^i - \Phi_i\right)x_ix_i^{\top}) \left(\left[G_{t,k-1}^i\right]^{\top}G_{t,k-1}\right)\otimes \left(F_{k+1}^i\left[F_{k+1}^i\right]^{\top}\right) \mbox{vec}\left((W_t^i - \Phi_i)x_i x_i^\top  \right)\\
     \nonumber \le &  {\nm{\left[G_{t,k-1}^i\right]^{\top}G_{t,k-1}}_2\nm{F_{k+1}^i\left[F_{k+1}^i\right]^{\top}}_2 \nm{\left(W_t^i - \Phi_i\right)x_i}^2}\\
     \overset{(\ref{eqn:eig-upper-bound})}{\le} &  \beta^2 \nm{\left(W_t^i - \Phi_i\right)x_i}^2,
\end{align}
where for any matrices $A,B,X$: (a) uses $\text{Tr}(AB) = \text{Tr}(BA)$; (b) uses $\text{Tr}(AB) = \mbox{vec}^{\top}(B) \mbox{vec}(A^\top)$; (c) uses $(AB)^\top = B^\top A^\top$; (d) uses $\mbox{vec}(AXB) = B^\top\otimes A \mbox{vec}(X)$.

\begin{align}\label{eqn:lemm:descent:2:key:4}
\nonumber&\sum_{i=1}^n \sum_{j\ne i}\left\| G^i_{t,k-1} x_ix_i^\top(W_t^i - \Phi_i)^{\top} F^i_{t,k+1} \left[F^j_{t,k+1}\right]^{\top}(W_t^j - \Phi_j)x_j x_j^{\top} \left[G^j_{t,k-1}\right]^{\top}\right\|_2\\
\nonumber\overset{(a)}{=}&\sum_{i=1}^n \sum_{j\ne i}\left|x_i^\top(W_t^i - \Phi_i)^{\top} F^i_{t,k+1} \left[F^j_{t,k+1}\right]^{\top}(W_t^j - \Phi_j)x_j\right|\left\| G^i_{t,k-1} x_i x_j^{\top} \left[G^j_{t,k-1}\right]^{\top}\right\|_2\\
\nonumber\overset{(b)}{\le}&\sum_{i=1}^n \sum_{j\ne i}\left\|(W_t^i - \Phi_i)x_i\right\| \left\|F^i_{t,k+1} \left[F^j_{t,k+1}\right]^{\top}\right\|_2 \left\|(W_t^j - \Phi_j)x_j\right\| \left \| G^i_{t,k-1} x_i\right\|\left\| G^j_{t,k-1} x_j \right \|\\
\nonumber\overset{(\ref{eqn:cross-upper-bound})}{\le}&\gamma\beta^2\sum_{i=1}^n \sum_{j\ne i}\left\|(W_t^i - \Phi_i)x_i\right\|  \left\|(W_t^j - \Phi_j)x_j\right\| \\
\nonumber\overset{(c)}{\le}&\frac{\gamma\beta^2}{2}\sum_{i=1}^n \sum_{j\ne i}\left(\left\|(W_t^i - \Phi_i)x_i\right\|^2 + \left\|(W_t^j - \Phi_j)x_j\right\|^2\right) \\
\nonumber\le&n\gamma\beta^2\sum_{i=1}^n \left\|(W_t^i - \Phi_i)x_i\right\|^2  \\
=&  n\gamma\beta^2\ell(W_t),
\end{align}
where (a) uses $\|c A\|_2 = |c|\|A\|_2$ for a matrix $A$ and a scalar $c$; (b) uses the facts that $|x^\top A y| \le \|x\|\|Ay\|\le \|x\| \|A\|_2 \|y\|$ and $\|x y^\top\|_2 = \|x \otimes y^\top\|_2 = \|x\|\|y\|$; (c) uses Young's inequality. Therefore, by (\ref{eqn:lemm:descent:2:key:2}) (\ref{eqn:lemm:descent:2:key:3}) and (\ref{eqn:lemm:descent:2:key:4}) we have
\begin{align}\label{eqn:lemm:descent:2:key:5}
\nm{\nabla_k\ell(W_t)}_2^2 \le (\beta^2 + n\gamma\beta^2)\ell(W_t)
\end{align}
We can rewrite $\Delta_{t,i}$ in (\ref{eqn:delta}) as
\begin{align*}
\Delta_{t,i} 
=& \sum_{s=2}^L(-\eta)^{s-2}\sum_{L\ge k_1 > k_2>\ldots>k_s \ge 1} F_{t,k_1+1}^i\left(\prod_{\ell=1}^s \nabla_{k_{\ell}}\ell(W_t)Z^{t,i}_{k_{\ell}-1,k_{\ell+1}}\right)G^i_{t,k_s-1}, {\text{~where~} Z_{k_s-1,k_{s+1}}^{t,i} = I.}
\end{align*}
Therefore, according to Lemma~\ref{lemma:z-ka-kb-t}, for any $\theta \in (0, 1/2)$,  with a probability $ 1- 4L^2\exp(-\theta^2 m/[16L^2])$, for any $1 \leq k_b < k_b \leq L$, we have
\[
    \|Z_{k_a, k_b}^t\|_2 \leq 4\sqrt{L}  e^{\theta/2}\theta^{-1/2}
\]

with a probability $1 - L^2\sqrt{m}\exp\left(-\theta m/[4L] + 6\sqrt{m}\right)$
\begin{align}\label{eqn:lemm:descent:2:key:6}
\nonumber\|\Delta_{t,i}\|_2 
\nonumber\overset{(a)}{\le} & \sum_{s=2}^L \eta^{s-2}\sum_{L\ge k_1 > k_2>\ldots>k_s \ge 1} \|F_{t,k_1+1}^i\|_2\|G^i_{t,k_s-1}\|_2 \left(\prod_{\ell=1}^s \|\nabla_{k_{\ell}}\ell(W_t)\|_2\|Z^{t,i}_{k_{\ell}-1,k_{\ell+1}}\|_2\right)\\
\nonumber\overset{(b)}{\le} & \sum_{s=2}^L \eta^{s-2} (_s^L) \beta \left( \sqrt{(\beta^2 + n\gamma\beta^2)\ell(W_t)} \times 4\sqrt{L}  e^{\theta/2}\theta^{-1/2}\right)^s\\
\nonumber\overset{(c)}{\le} & \sum_{s=2}^L \beta \eta^{s-2} \left( 2\sqrt{2}\sqrt{L}  e^{\theta/2}\theta^{-1/2}\beta \sqrt{\ell(W_t)}\right)^s\\
\nonumber\overset{(d)}{\le} & 8L  e^{\theta}\theta^{-1}   \beta^3\ell(W_t)\sum_{s=0}^{L-2} \eta^{s} \left(2\sqrt{2}\sqrt{L}  e^{\theta/2}\theta^{-1/2}\beta \sqrt{\ell_0}\right)^s\\
\overset{(e)}{\le} & \frac{8L  e^{\theta}\theta^{-1}   \beta^3\ell(W_t)}{1 -  2\sqrt{2}\eta\sqrt{L}  e^{\theta/2}\theta^{-1/2}\beta \sqrt{\ell_0}}
\end{align}
where (a) uses $\|AB\|_2\le\|A\|_2\|B\|_2$ and $\|A+B\|_2\le\|A\|_2+\|B\|_2$ for any matrices $A,B\in\R^{m\times m}$; (b) uses (\ref{eqn:eig-upper-bound}), (\ref{eqn:lemm:descent:2:key:5}) and Lemma~\ref{lemma:z-ka-kb-t} that $\|Z_{k_a, k_b}^t\|_2 \leq 5e^{\theta/2}m^{1/4}$; (c) uses the facts that $(_s^L) = \frac{L!}{(L-s)!s!}\le\frac{L!}{(L-s)!} = L(L-1)\dots(L-s+1)\le L^s$ and (\ref{eqn:gamma-upper-bound}) with $\alpha\le\beta$; (d) uses $\ell(W_t)\le\ell_0$; (e) uses $\eta<1/(2\sqrt{2}\sqrt{L}  e^{\theta/2}\theta^{-1/2}\beta \sqrt{\ell_0})$.
Therefore, by (\ref{eqn:lemm:descent:2:key:1}) and (\ref{eqn:lemm:descent:2:key:6}) we have
\begin{align*}
    \ell(W_{t+1}) - \ell(W_t) \leq -\frac{3\eta\alpha^2 L}{4}\ell(W_t) + \frac{128 n \eta^2  e^{2\theta}\theta^{-2}   \beta^2\ell^2(W_t)}{1 -  2\sqrt{2}\eta\sqrt{L}  e^{\theta/2}\theta^{-1/2}\beta \sqrt{\ell_0}}.
\end{align*}
We complete the proof by using the selection of $\eta$.
\end{proof}

\subsection{Proof for Theorem~\ref{thm:linear_rate_convergence}}
\begin{proof} \label{proof:thm_linear_rate}
We start by restricting the network's width and depth to values for which assumption~(\ref{eqn:gamma-upper-bound}) holds,
that is $\beta^2 \gamma n  \leq \alpha^2/2$. Revisiting~(\ref{eqn:gamma_value}), we need condition,
\begin{eqnarray*}
    \gamma = C''\left(\left(L^{3/2}\tau \right)^2 + \delta\left(\frac{5}{6}\right)^L + \frac{1}{m^{1/2}}\right)=O\left(\frac{1}{n} \right) 
\end{eqnarray*}
To meet the above condition, all the following should hold,
\begin{eqnarray} \label{eqn:conditions_gamma}
    L^{3/2}\tau = O\left(\frac{1}{\sqrt{n}}\right), \quad L=\Omega(\log n), \quad m = \Omega(n^2).
\end{eqnarray}
We thus need to bound $\tau$. According to the setting, we know the number of iterations is
\begin{eqnarray} \label{eqn:niters}
    T = \frac{2}{\eta\alpha^2 L}\log\frac{\ell_0}{\epsilon}
\end{eqnarray}
On the other hand,
\begin{align*}
    \tau \le \eta \sum_{t=1}^{T-1}\max_{k\in[L]}\|\nabla_k\ell(W_t)\| 
    \le \eta \beta  \sum_{t=1}^{T-1} \sqrt{2\ell(W_t)}
    \le  \eta \beta  T \sqrt{2\ell_0}
    = \frac{2 \beta  \sqrt{2\ell_0}}{\alpha^2 L}\log\frac{\ell_0}{\epsilon}
\end{align*}
where the second inequality uses (\ref{eqn:lemm:descent:2:key:5}) which assumes (\ref{eqn:grad}) and (\ref{eqn:cross-upper-bound}). 
Similarly to (\ref{eqn:min_max_G}), the initial output of the network can be bounded with high probability in order to bound $\ell_0$. With a probability $ 1- 4L^2\exp(-\theta^2 m/[8L^2] +3d_x+3d_y)$
\begin{eqnarray} \label{eqn:initial_loss}
\ell_0=\frac{1}{2}\sum_{i=1}^n \nm{(W_0^i - \Phi_i)x_i}^2
\leq \frac{1}{2}\sum_{i=1}^n\paren{\nm{(W_0^ix_i}^2+\nm{y_i}^2 } = \frac{n}{2}\paren{\frac{3m}{d_x}e^{\theta}+\frac{m}{d_x}}
\leq \frac{4m n}{d_x}
\end{eqnarray}
Where we also used Assumption~(\ref{assum:labels}) and $\theta<0.5$. We now extract the required width from (\ref{eqn:conditions_gamma}),

\begin{eqnarray} \label{eqn:width_extract}
L^{3/2}\tau 
=\tilde{O}\left(\frac{\sqrt{\ell_0 d_x d_y L}}{m} \right)=O\left(\frac{1}{\sqrt{n}}\right)
\end{eqnarray}

Therefore from (\ref{eqn:width_extract},\ref{eqn:initial_loss}) 
we get,
\begin{eqnarray} \label{eqn:min_width}
m=\tilde{\Omega}\paren{n^2 L d_y}
\end{eqnarray}
Considering the learning rate, revisiting~(\ref{eqn:eta_min}), 
\begin{eqnarray*}
    \eta &=& \min\left( \frac{ \alpha^2}{12\beta^2\beta_xL}, \frac{1}{3 L}, \frac{\alpha^2}{ 4\beta^4 L}, \frac{1}{ \beta^2 L}, \frac{1}{4\sqrt{2}\sqrt{L}  e^{\theta/2}\theta^{-1/2}\beta \sqrt{\ell_0}}, \frac{\alpha^2}{1024 n e^{2\theta}\theta^{-2}   \beta^2\ell_0}\right) \\
    &=& \min\left( O\paren{\frac{d_x^{1/2}}{m^{1/2}L}},
                   O\paren{\frac{1}{L}},
                   O\paren{\frac{d_x d_y}{m^2 L}},
                   O\paren{\frac{d_x d_y}{m^2 L}},
                   O\paren{\frac{d_x d_y^{1/2}}{m^{3/2}n^{1/2}L^{1/2}}},
                   O\paren{\frac{d_x}{m n^2}}\right) \\
   &\overset{(\ref{eqn:min_width})}{=}& \min\left( \tilde{O}\paren{\frac{d_x^{1/2}}{n L^{3/2}d_y^{1/2}}},
                   \tilde{O}\paren{\frac{1}{L}},
                   \tilde{O}\paren{\frac{d_x}{n^4 L^3 d_y}},
                   \tilde{O}\paren{\frac{d_x}{n^{7/2} L^2 d_y}},
                   \tilde{O}\paren{\frac{d_x}{n^4 L d_y}} \right)\\
       &=& \tilde{O}\paren{\frac{d_x}{n^4 L^3 d_y}}
\end{eqnarray*}
So, the learning rate is therefore $\frac{1}{\beta^2 L}$ 
and the corresponding number of iterations in~(\ref{eqn:niters}),
\begin{eqnarray}
T=c_T\log\frac{\ell_0}{\epsilon} =c_T\log\paren{\frac{n^3L}{d_x\epsilon}} \quad,\quad c_T=2\cdot 81^2
\end{eqnarray}
\end{proof}

\subsection{Proof for Theorem~\ref{thm:NTK}} \label{proof:thm_NTK}
\begin{proof}
We start with the left-hand side, and denote for brevity $D^i_k=D_k$. We will bound for every $k\in[L]$ independently, and the final result follows, 
\begin{eqnarray*}
 \nm{\nabla_{W_k} y_p(x_i,\bar{W})-\nabla_{W_k} y_p^{\textrm{NTK}}(x_i,W_1)}_F &\leq& \nm{B}_F\nm{(Z^i_{L,k}+\Sigma^i_{L,k})(Z^i_{k-1,1}+\Sigma^i_{k-1,1})-Z^i_{L,k}Z^i_{k-1,1}}_F\nm{C}_F \\ 
 &\leq& \nm{B}_F(\nm{Z^i_{L,k}\Sigma^i_{k-1,1}}_F+\nm{Z^i_{k-1,1}\Sigma^i_{L,k}}_F+\nm{\Sigma^i_{k-1,1}\Sigma^i_{L,k}}_F)\nm{C}_F 
\end{eqnarray*}
Where we used the sub-multiplicativity of Frobenius norm, with, 
\begin{eqnarray*}
Z^i_{k,k'} &=& D_k^iW_{1,k}\ldots W_{1,k'+1}D_{k'}^i \\
\Sigma^i_{k,k'} &=& \sum_{s=1}^{k- k'} \sum_{k_1 > k_2 >\ldots > k_s}Z^i_{k, k_1+1}\prod_{\ell=1}^s W'_{k_{\ell}}Z^i_{k_{\ell}+1, k_{\ell+1}}
\end{eqnarray*}
According to Lemma~\ref{lemma:z-ka-kb-2}, with a probability larger than $1- 81L^3\exp(- m/[108L])$, we have for all $k$,
\begin{eqnarray*}
\nm{Z^i_{k-1,1}\Sigma^i_{L,k}}_2 
\leq 19\sqrt{L}\sum_{s=1}^{k-k'} \left[19L^{3/2}\xi\right]^s \leq
\frac{(19L)^2\xi}{1-19L^{3/2}\xi}.   
\end{eqnarray*}
Similar analysis can be applied to the other term, therefore,
\begin{eqnarray*} 
(\nm{Z^i_{L,k}\Sigma^i_{k-1,1}}_F+\nm{Z^i_{k-1,1}\Sigma^i_{L,k}}_F+\nm{\Sigma^i_{k-1,1}\Sigma^i_{L,k}}_F) \leq 3\nm{Z^i_{k-1,1}\Sigma^i_{L,k}}_F 
\end{eqnarray*}
The terms $\nm{B}^2_F$,$\nm{C}^2_F$ follow independant $\frac{2}{d_y}\chi^2_{m d_y}, \frac{2}{d_x}\chi^2_{m d_x}$ distributions. Therefore by Laurent-Massart concentration bound~\citep{laurent2000adaptive},
\begin{eqnarray}
\Pr\left(\nm{B}_F \geq \sqrt{5m}\right) \leq 2\exp\left(-\frac{m d_y}{4}\right) \quad, \quad \Pr\left(\nm{C}_F \geq \sqrt{5m}\right) \leq 2\exp\left(-\frac{m d_x}{4}\right)
\end{eqnarray}

We can conclude, with a probability of at least $1-\exp(-\Omega(m))$,
\begin{eqnarray*}
\nm{\nabla_{W_k} y_p(x_i,\bar{W})-\nabla_{W_k} y_p^{\textrm{NTK}}(x_i,W_1)}_F \leq O\left(m L^2\xi\right)= \tilde{O}\left((m n)^{1/2}d_y L^2\right)
\end{eqnarray*}
This term grows with $m$ due to the unique normalization of the network. We now show that it is negligible with respect to the NTK gradient in the right-hand side. Using theorem~\ref{thm:bound_alpha_beta}, probability of at least $1-\exp(-\Omega(m))$,
\begin{eqnarray*}
\nm{\nabla_k y_p^{\textrm{NTK}}(x_i,W_1)}^2_F=\nm{\left[G^i_{1,k-1} F^i_{t,k+1}\right]^{\top}}_F^2&=&\Tr\left( G^i_{1,k-1}F^i_{t,k+1} \left[G^i_{1,k-1}F^i_{t,k+1}\right]^{\top} \right) \\
&=& \Tr\left( \left[G^i_{1,k-1}\right]^{\top} G^i_{1,k-1}F^i_{t,k+1}\left[F^i_{t,k+1}\right]^{\top}\right) \\
&=& \sum_{i=1}^m \lambda\left( \left[G^i_{1,k-1}\right]^{\top} G^i_{1,k-1}F^i_{t,k+1}\left[F^i_{t,k+1}\right]^{\top}\right)  \\
&\geq&  m\lambda_{\min}\left( \left[G^i_{1,k-1}\right]^{\top} G^i_{1,k-1} \right)
\lambda_{\min}\left( F^i_{t,k+1}\left[F^i_{t,k+1}\right]^{\top}\right) \\
&=& m\alpha^2 \\
&=&\frac{m^3}{144 d_x d_y}
\end{eqnarray*}
Where we used the inequality of arithmetic and geometric means and multiplicativity of determinant for symmetric and positive semi-definite $A,B$, i.e,  
\begin{eqnarray*}
\sum_{i=1}^m \lambda_i\left( AB \right) \geq m\sqrt[\leftroot{-3}\uproot{3}m]{\prod_{i=1}^m \lambda_i\left( AB \right)} =m\sqrt[\leftroot{-3}\uproot{3}m]{\prod_{i=1}^m \lambda_i\left( A \right)\lambda_i\left( B \right)} \geq m\sqrt[\leftroot{-3}\uproot{3}m]{\lambda^m_{\min}\left( A \right)\lambda^m_{\min}\left( B \right)}
=m\lambda_{\min}\left( A \right)\lambda_{\min}\left( B \right)
\end{eqnarray*}
Combining both terms completes the proof. 
\end{proof}

\subsection{Proof for Corollary~\ref{thm:NTK_corollary}} \label{proof:thm_NTK_cor}
\begin{proof}
We will apply Theorem~\ref{thm:NTK} for 2 examples $(x_i,x_j)$ to validate the given bound. Notice the following,
\begin{eqnarray*}
\begin{cases}
\langle \nabla y_p^{\textrm{NTK}}(x_i,W_1), \nabla y_p^{\textrm{NTK}}(x_i,W_1)\rangle &=  ~~K_p^{\textrm{NTK}}(x_i, x_i) \\
\langle \nabla y_p^{\textrm{NTK}}(x_i,W_1), \nabla y_p^{\textrm{NTK}}(x_j,W_1)\rangle &=  ~~K_p^{\textrm{NTK}}(x_i, x_j) \\
\end{cases}
\end{eqnarray*}
We start from the left hand side of equation~(\ref{eqn:NTK_cor_main}),
\begin{eqnarray*}
&&|\langle\nabla y_p(x_i,\bar{W}),\nabla y_p(x_j,\bar{W})\rangle - 
K_p^{\textrm{NTK}}(x_i, x_j)| \\
&=&
|\langle\nabla y_p(x_i,\bar{W}),\nabla y_p(x_j,\bar{W})\rangle 
-\langle \nabla y_p^{\textrm{NTK}}(x_i,W_1),\nabla y_p(x_j,\bar{W})\rangle \\
&&+\langle \nabla y_p^{\textrm{NTK}}(x_i,W_1),\nabla y_p(x_j,\bar{W})\rangle
- \langle \nabla y_p^{\textrm{NTK}}(x_i,W_1), \nabla y_p^{\textrm{NTK}}(x_j,W_1)\rangle| \\
&\leq& \nm{\nabla y_p(x_j,\bar{W})}_F\nm{\nabla y_p(x_i,\bar{W})-\nabla y_p^{\textrm{NTK}}(x_i,W_1)}_F \\
&&+ \nm{\nabla y_p^{\textrm{NTK}}(x_i,W_1)}_F \nm{\nabla y_p(x_j,\bar{W})-\nabla y_p^{\textrm{NTK}}(x_j,W_1)}_F \\
&\leq& \mathcal{R} \nm{\nabla y_p^{\textrm{NTK}}(x_i,W_1)}_F 
\paren{2\nm{\nabla y_p^{\textrm{NTK}}(x_j,W_1)}_F +\nm{\nabla y_p(x_j,\bar{W})-\nabla y_p^{\textrm{NTK}}(x_j,W_1)}_F } \\
&\leq& \paren{2\mathcal{R}+\mathcal{R}^2}\sqrt{K_p^{\textrm{NTK}}(x_i, x_i) K_p^{\textrm{NTK}}(x_j, x_j)}
\end{eqnarray*}
Since $\mathcal{R} \ll 1$, $2\mathcal{R}+\mathcal{R}^2<3\mathcal{R}$
\end{proof}
\subsection{Proof for Theorem~\ref{thm:equivalence2}} \label{proof:relu_equivalence}
\begin{proof} 

In order to obtain a GReLU network from a ReLU network, we proceed as follows.
First, let us define the following notations:
\begin{itemize}
    \item $x_i$ denotes the $i^\text{th}$ training example.
    \item ${z^\text{ReLU}}^i_{k}$ denotes the output feature map of the $k^\text{th}$ layer taken \textbf{before} the ReLU in the ReLU network on the $i^\text{th}$ training example, and is recursively defined as: 
    $$
    {z^\text{ReLU}}^i_{k}=\tilde{W}_k\text{ReLU}\left({z^\text{ReLU}}^i_{k-1}\right)
    $$
    \item ${\tilde{z}}^i_{k}$ denotes the output feature map of the $k^\text{th}$ layer taken \textbf{after} the GReLU activation in the GReLU network on the $i^\text{th}$ training example, and is recursively defined as:  
    $$
    {\tilde{z}}^i_{k}=\text{GReLU}\left(W_{t,k} {\tilde{z}}^i_{k-1},\Psi_i {z}^i_{k-1}\right)
    $$
    \item ${z}^i_{k}$ represents the feature map of the auxiliary network activating the $k^\text{th}$ GReLU gate on the $i^\text{th}$ training example, and is recursively defined as: 
    $$
    {z}^i_{k} = \Psi_i {z}^i_{k-1}
    $$
    
\end{itemize}
Matching the intermediate layer of the GReLU and ReLU network can be written as:
$$
{z^\text{ReLU}}^i_{k} = {\tilde{z}}^i_{k}~~\forall k \in [L]~~\forall i \in [n].
$$
that is
$$
\tilde{W}_k\text{ReLU}\left({z^\text{ReLU}}^i_{k-1}\right)=\text{GReLU}\left(W_{t,k} {\tilde{z}}^i_{k-1},\Psi_i {z}^i_{k-1}\right) ~~\forall k \in [L]~~\forall i \in [n],
$$
that is equivalent to
\begin{equation}
\label{eqn:GRELU_RELU}
\tilde{W}_k\text{ReLU}\left({z^\text{ReLU}}_{k-1}\right)=\text{GReLU}\left(W_{t,k} {\tilde{z}}_{k-1},\Psi_i {z}_{k-1}\right) ~~\forall k \in [L].
\end{equation}

Hence we proceed by recursion. Given that
$$
{z^\text{ReLU}}_{k-1} = {\tilde{z}}_{k-1},
$$
we seek for $\tilde{W}_k$ that obeys equation (\ref{eqn:GRELU_RELU}). Denoting by $(\cdot)^\dagger$ the Moore–Penrose pseudo inverse, and by setting
\begin{equation}
\label{eqn:GRELU_RELU2}
\tilde{W}_k=\left(\text{GReLU}\left(W_{t,k} {\tilde{z}}_{k-1}, \Psi_i {z}_{k-1}\right)\right)^\dagger {z^\text{ReLU}}_{k-1}, 
\end{equation}
equation (\ref{eqn:GRELU_RELU}) holds such that we have 
$$
{z^\text{ReLU}}_{k} = {\tilde{z}}_{k}.
$$
The equivalence between (\ref{eqn:GRELU_RELU}) and (\ref{eqn:GRELU_RELU2}) lays in the fact that the network is overparameterized such that $n\leq m$. In this case $\tilde{W}_k$ has dimension $(m\times m)$ and $\text{GReLU}\left(W_{t,k} {\tilde{z}}_{k-1}, \Psi_i {z}_{k-1}\right)$ has dimension $m\times n$, and thus, (\ref{eqn:GRELU_RELU}) and (\ref{eqn:GRELU_RELU2}) are equivalent.

\end{proof}

\section{Main Theorems with Proofs}
\begin{thm} \label{thm:bound_alpha_beta}
Choose $\theta \in (0, 1/2)$ which satisfies
\begin{eqnarray}
    L\tau 3 \sqrt{12 L}  e^{2 \theta}\theta^{-1/2} \leq \frac{1}{9} \label{eqn:cond-tau}
\end{eqnarray}
With a probability at least $1 - 3L^2\sqrt{m}\exp\left(-\theta m/[4L] + 6\sqrt{m} + 3\max\{d_x,d_y\}\right)$, we have, for any $t$, $k$ and $i$
\[
    \left\|\left[G_{t,k}^i\right]^{\top}G_{t,k}^i\right\|_2 \leq \frac{27m}{4d_x}, \quad     \left\|F_{t,k}^i\left[F_{t,k}^i\right]^{\top}\right\|_2 \leq \frac{27m}{4d_y}
\]
and
\[
    \lambda_{\min}\left(\left[G_{t,k}^i\right]^{\top}G_{t,k}^i\right) \geq \frac{m}{12d_x}, \quad     \lambda_{\min}\left(F_{t,k}^i\left[F_{t,k}^i\right]^{\top}\right) \geq \frac{m}{12d_y}
\]
Thus, with a high probability, we have
\[
    \alpha_x = \frac{m}{12d_x}, \alpha_y = \frac{m}{12d_y}, \quad \beta_x = \frac{27m}{4d_x}, \beta_y = \frac{27m}{4d_y}.
\]
\end{thm}

\begin{proof} \label{proof:thm_bound_alpha_beta}
We will focus on $G^i_{t,k}$ first and similar analysis will be applied to $F^i_{t,k}$. To bound $\nm{\left[G^i_{t,k}\right]^{\top}G^i_{t,k}}_2$,
we will bound $\max\limits_{\nm{v} \leq 1} \nm{G^i_{t,k}v}^2$. Define $\delta W_{t,k} := W_{t,k} - W_{1,k}$, we have
\begin{align*}
    G^i_{t,k} =& D^i_k\left(W_{1,k} + \delta W_{t,k}\right)\ldots D^i_1\left(W_{1,1} + \delta W_{t,1}\right)D_0^iC \\
    =& D^i_kW_{1,k}\ldots D^i_0C\\
    &+ \sum_{s=1}^k \sum_{k_1 > k_2 > \ldots > k_s} D^i_kW_{1,k}\ldots D^i_{k_1}\delta W_{t,k_1}D^i_{k_1-1}W_{1,k_1-1}\ldots D^i_{k_s}\delta W_{t,k_s}D^i_{k_s-1}W_{1,k_s-1}\ldots D^i_0C\\
    =& Z_{k,0}^{1,i}C + \sum_{s=1}^k \sum_{k_1 > k_2>\ldots>k_s} \left(\prod_{j=1}^{s} Z_{k_{j-1}, k_j}\delta W_{t,k_j}\right)Z_{k_s-1,0}C.
\end{align*}
According to Lemma~\ref{lemma:z-ka-kb-2}, we have, with a probability {$ 1- 2L\exp(-\theta^2 m/[16L^2] )$}, for any $1 \leq k_b \leq k_a \leq L$ and any $\theta \in (0,1/2)$,
\[
    \|Z_{k_a, k_b}^{1,i}\|_2 \le \sqrt{12 L}  e^{\theta/2}\theta^{-1/2}
\]
Following the same analysis as in Lemma~\ref{lemma:z-ka-kb-2}, we have, with a probability $ 1- 4L^2\exp(-\theta^2 m/[8L^2] +3d_x)$, we have
\begin{eqnarray} \label{eqn:min_max_G}
\min\limits_{u \in \R^{d_x}}\frac{\nm{Z^{1,i}_{k_a, k_b} C u}}{\nm{u}} \geq \sqrt{\frac{m}{3d_x}}e^{-\theta}\quad
\textrm{and} \quad
\max\limits_{u \in \R^{d_x}}\frac{\nm{Z^{1,i}_{k_a, k_b} C u}}{\nm{u}} {\leq} \sqrt{\frac{3m}{d_x}}e^{\theta/2}
\end{eqnarray}
Combining the above results together, we have, with a probability $ 1- 4L^3\exp(-\theta^2 m/[16L^2]+3d)$,
\[
\min\limits_{u \in \R^{d_x}}\frac{\nm{G_{t,k}^iu}}{\nm{u}} \geq \sqrt{\frac{m}{3d_x}}e^{-\theta} - \sqrt{\frac{3m}{d_x}}e^{\theta/2}\sum_{s=1}^L \left(L\tau \sqrt{12 L}  e^{\theta/2}\theta^{-1/2}\right)^s \geq  \sqrt{\frac{m}{3d_x}}e^{-\theta}\left(1 - \frac{L\tau 3 \sqrt{12 L}  e^{2 \theta}\theta^{-1/2}}{1 - L\tau \sqrt{12 L}  e^{\theta/2}\theta^{-1/2}}\right)
\]
and
\[
\max\limits_{u \in \R^{d_x}}\frac{\nm{G_{t,k}^i u}}{\nm{u}} \leq \sqrt{\frac{3m}{d_x}}e^{\theta/2} + \sqrt{\frac{3m}{d_x}}e^{\theta/2}\sum_{s=1}^L \left(L\tau \sqrt{12 L}  e^{\theta/2}\theta^{-1/2}\right)^s \leq \sqrt{\frac{3m}{d_x}}e^{\theta/2}\left(1 + \frac{L\tau \sqrt{12 L}  e^{\theta/2}\theta^{-1/2}}{1 - L\tau \sqrt{12 L}  e^{\theta/2}\theta^{-1/2}}\right)
\]
Since
\[
    L\tau 3 \sqrt{12 L}  e^{2 \theta}\theta^{-1/2} \leq \frac{1}{9}
\]
we have, with a probability $1- 4L^3\exp(-\theta^2 m/[16L^2]+3d_x)$,
\[
\min\limits_{u \in \R^{d_x}}\frac{\nm{G_{t,k}^iu}_2}{\nm{u}} \geq \frac{1}{2}\sqrt{\frac{m}{3d_x}}, \quad \max\limits_{u \in \R^{d_x}}\frac{\nm{G_{t,k}^iu}}{\nm{u}} \leq \frac{3}{2}\sqrt{\frac{3m}{d_x}}
\]
Similar analysis applies to $\nm{F_{t,k}^i}_2$.
\end{proof}

\begin{thm}\label{thm:bound_gamma}
Assume $L^{3/2}\tau < 1$. With a probability $1 - (4L^2+n^2)\exp(-\Omega(\sqrt{m} + \max\{d_x,d_y\}))$, we have
\begin{eqnarray*}
\nm{F_{t,k}^j\left[F_{t,k}^i\right]^{\top}}_2 & \leq & C'\left(\frac{1}{m^{1/4}} + \left(\frac{5}{6}\right)^{L-k} + L^{3/2}\tau\right)\beta_y \\
\left|\langle G_{t,k}^jx_j, G_{t,k}^i x_i\rangle\right| & \leq & C'\left(\frac{1}{m^{1/4}} + \delta\left(\frac{5}{6}\right)^{k} + L^{3/2}\tau\right)\beta_x
\end{eqnarray*}
where $C'$ is an universal constant, and therefore
\[
\left|\left\langle G_{t,k-1}^jx_j, G_{t,k-1}^ix_i\right\rangle\right|\nm{F^j_{t,k+1}\left[F_{t,k+1}^i\right]^{\top}}_2 \leq \gamma\beta^2
\]
with
\begin{eqnarray} \label{eqn:gamma_value}
    \gamma = C''\left(L^{3}\tau^2 + \delta\left(\frac{5}{6}\right)^L + \frac{1}{m^{1/2}}\right)
\end{eqnarray} 
\end{thm}

\begin{proof} \label{proof:thm_bound_gamma}
To bound $\nm{F_{t,k}^j\left[F_{t,k}^i\right]^{\top}}_2$, we define $k_0 = L+1$
and express $F_{t,k}^j\left[F_{t,k}^i\right]^{\top}$ as
\begin{eqnarray*}
\lefteqn{F_{t,k}^j\left[F_{t,k}^i\right]^{\top}} \\
& = & \underbrace{BD^j_LW_{1,L}\ldots W_{1,k+1} D^j_kD_k^iW_{1,k+1}^{\top}\ldots W_{1,L}^{\top}D^i_LB^{\top}}_{:=Z} \\
&   & + \underbrace{BZ_{L:k}^j\sum_{s=1}^{L-k} \left[\left(\prod_{\ell=1}^{s}Z_{k_{\ell-1}-1:k_{\ell}}^{i}\delta W_{t,k_{\ell}}\right) Z^i_{k_{s}-1:k}\right]^{\top}B^{\top}}_{:=\A_1} \\
& &+ \underbrace{B\sum_{s=1}^{L-k} \left[\left(\prod_{\ell=1}^{s}Z_{k_{\ell-1}-1:k_{\ell}}^{j}\delta W_{t,k_{\ell}}\right)Z^j_{k_{s}-1:k}\right]\left[Z_{L:k}^i\right]^{\top} B^{\top}}_{:=\A_2} \\
&   & + \underbrace{B\sum_{s_a=1}^{L-k}\sum_{s_b=1}^{L-k} \left(\prod_{\ell=1}^{s_a}Z_{k_{\ell-1}:k_{\ell}}^{j}\delta W_{t,k_{\ell}}\right)Z^j_{k_{s_a}-1:k}\left[Z^i_{k_b-1:k}\right]^{\top}\left[\prod_{\ell=2}^{s_b}Z_{k_{\ell-1}:k_{\ell}}^{i}\delta W_{t,k_{\ell}}\right]^{\top}B^{\top}}_{:=\A_3}
\end{eqnarray*}
According to Lemma~\ref{lemma:z-ka-kb-2}, with a probability $ 1- 4L^3\exp(-\theta^2 m/[16L^2]+3d_y)$,
\[
    \|Z_{k_a, k_b}\|_2 \leq \sqrt{12 L}  e^{\theta/2}\theta^{-1/2}, \quad \max\limits_{\|v\| \leq 1} \left\|v^{\top}BZ_{k_a,k_b} \right\| \leq \frac{3}{2}\sqrt{\frac{3m}{d_y}}
\]
As a result, with a probability $ 1- 4L^3\exp(-\theta^2 m/[16L^2]+3d_y)$, we have
\begin{eqnarray*}
\max\limits_{\|v\| \leq 1} v^{\top}\A_1 v \leq \frac{27m}{4d}\sum_{s=1}^L \left(\sqrt{12 L}  e^{\theta/2}\theta^{-1/2}L\tau\right)^s \leq \xi \frac{m}{d_y}
\end{eqnarray*}
where $\xi = 8\sqrt{12 L}  e^{\theta/2}\theta^{-1/2}L\tau \ll 1$ and the last step utilizes the assumption $\sqrt{12 L}  e^{\theta/2}\theta^{-1/2} L\tau \leq 1/9$.
Using the same analysis, we have, with a probability $ 1- 4L^3\exp(-\theta^2 m/[16L^2]+3d_y)$,
\[
    \max\limits_{\|v\| \leq 1} v^{\top}\A_2v \leq \xi\frac{m}{d_y}, \quad \max\limits_{\|v\| \leq 1} v^{\top}\A_3v \leq \xi^2\frac{m}{d_y}
\]
Next we need to bound the spectral norm of $Z = F_{1,k}^j\left[F_{1,k}^i\right]$. To this end, we will bound
\[
    \max\limits_{u\in \R^{d_y}, v\in \R^{d_y}, \|u\| =1, \|v\| = 1} \left\langle \left[F^j_{1,k}\right]^{\top}u, \left[F^i_{1,k}\right]^{\top}v\right\rangle
\]
Define
\[
    a_k = \left[F^j_{t,k+1}\right]u, \quad b_k = \left[F^i_{t,k+1}\right]u
\]
We have
\begin{eqnarray}
\lefteqn{a^{\top}_{k} b_{k} = a_{k+1}^{\top}W_{1,k+1}D_k^j D_k^i W^{\top}_{1,k+1} b_{k+1}} \nonumber \\
& = & \underbrace{a_{k+1}^{\top}W_{1,k+1}D_k^j D_k^i W^{\top}_{1,k+1} P_{a_{k+1}}\left(b_{k+1}\right)}_{:=Z_1} + \underbrace{a_{k+1}^{\top}W_{1,k+1}D_k^j D_k^i W^{\top}_{1,k+1}\left(b_{k+1} - P_{a_{k+1}}\left(b_{k+1}\right)\right)}_{:=Z_2} \nonumber
\end{eqnarray}
where $P_a(b) = aa^{\top}b/\nm{a}^2$ projects vector $b$ onto the direction of $a$. Since $W^{\top}_{1,k+1}\left(b_{k+1} - P_{a_{k+1}}(b_{k+1})\right)$ is statistically independent from $W_{1,k+1}^{\top}a_{k+1}$, by fixing $W_{1,k+1}^{\top}a_{k+1}$ and viewing $Z_2$ as a random variable of $W^{\top}_{1,k+1}\left(b_{k+1} - P_{a_{k+1}}(b_{k+1})\right)$, we have $Z_2$ follow $\N(0, \sigma^2)$ distribution, where
\[
    \sigma^2 = \frac{1}{m}\nm{b_{k+1} - P_{a_{k+1}}(b_{k+1})}^2\nm{D_k^jW_{1,k+1}^{\top}a_{k+1}}^2
\]
Using the property of Gaussian distribution, we have, with a probability $1 - 2\exp(-\sqrt{m}/2)$,
\[
\left|Z_2\right| \leq \frac{\nm{b_{k+1}}}{m^{1/2}}\nm{D_k^jW_{1,k+1}^{\top}a_{k+1}}
\]
Since $\frac{m \nm{D_k^jW_{1,k+1}^{\top}a_{k+1}}^2}{2|a_{k+1}|^2} $ follows $\chi^2_{m/2}$, with a probability $1 - \exp(-\theta m/4)$, we have
\[
\nm{D_k^jW_{1,k+1}^{\top}a_{k+1}}^2 \leq (1 + \theta)\nm{a_{k+1}}^2
\]
Combining the above two results, with a probability $1 - 2\exp(-\sqrt{m}/2) - \exp(-\theta m/4)$,
\[
\left|Z_2\right| \leq \frac{1+\theta}{m^{1/2}}\nm{b_{k+1}}\nm{a_{k+1}} \leq \frac{(1+\theta)\beta}{m^{1/2}}
\]
where the last step utilizes the fact $\nm{b_{k+1}} \leq \sqrt{\beta}$ and $\nm{a_{k+1}} \leq \sqrt{\beta}$.

To bound $Z_1$, let $\I_k^{i,j}$ be the set of indices of diagonal elements that are nonzero in both $D_k^j$ and $D_k^i$. According to Lemma~\ref{lemma:1}, with a probability $1 - \exp(-\Omega(m))$, we have $|\I_k^{i,j}| \leq m/3$. Define $u = W^{\top}_{1,k+1}a_{k+1}/\nm{a_{k+1}}$. We have
\[
    Z_1 = \left|a_{k+1}^{\top}b_{k+1}\right|\sum_{s \in \I_k^{i,j}} u_s^2 \leq \left|a_{k+1}^{\top}b_{k+1}\right|\sum_{s \in \I_k^{i,j}} u_s^2 
\]
Using the concentration of $\chi^2$ distribution, we have, with a probability $1 - \exp(-\theta m/3))$
\[
\sum_{s \in \I_k^{i,j}} u_i^2 \leq \frac{2(1+\theta)}{3}
\]
By combining the results for $Z_1$ and $Z_2$, for a fixed normalized $u\in \R^{d_y}$, $v \in \R^{d_y}$, with a probability $1 - \exp(-\Omega(\sqrt{m})$, we have
\[
    \left|\left\langle\left[F_{1,k}^j\right]^{\top}u, \left[F_{1,k}^i\right]^{\top}v\right\rangle\right| \leq \frac{(1+\theta)\beta}{m^{1/2}} + \frac{2(1+\theta)}{3}\left|\left\langle\left[F_{1,k+1}^j\right]^{\top}u, \left[F_{1,k+1}^i\right]^{\top}v\right\rangle\right|
\]
By combining the results over iterations, we have, with a probability $1 - 2L\exp(-\theta m/3) - 2L\exp\left(-\sqrt{m}/2\right)$, for fixed normalized $u \in \R^{d_y}$ and $v \in \R^{d_y}$
\[
\left|\left\langle\left[F_{t,k}^j\right]^{\top}u, \left[F_{t,k}^i\right]^{\top}v\right\rangle\right| \leq \frac{4(1+\theta)\beta}{m^{1/2}} + \left(\frac{2(1+\theta)}{3}\right)^{L-k}\beta
\]
We extend the result to any $u$ and $v$ by using the theory of covering number, and complete the proof by choosing $\theta = 1/6$.

To bound $|\langle G_{t,k}^j x_j, G_{t,k} x_i\rangle|$, follow the same analysis in above, we have
\[
\left|\langle G_{t,k}^j x_j, G^i_{t,k} x_i\rangle\right| \leq O\left(L^{3/2}\tau\beta\right) + \left|\langle G_{1,k}^j x_j, G^i_{1,k} x_i\rangle\right| \leq O\left(\left(L^{3/2} + \frac{1}{m^{1/2}}\right)\beta\right) + \left(\frac{5}{6}\right)^{k}\langle Cx_i, Cx_j\rangle
\]
We expand $\langle Cx_i, Cx_j\rangle$ as
\[
\langle Cx_i, Cx_j\rangle = \langle CP_{x_j}(x_i), Cx_j\rangle + \langle C\left(x_i - P_{x_j}(x_i)\right), Cx_j \rangle
\]
With a probability $1 - \exp(-\theta m/4)$, we have
\[
\left|\langle CP_{x_j}(x_i), Cx_j\rangle\right| \leq (1+\theta)\delta\beta
\]
With a probability $1 - 2\exp(-2\sqrt{m})$
\[
\left|\langle C\left(x_i - P_{x_j}(x_i)\right), Cx_j \rangle\right| \leq \frac{\beta}{m^{1/2}}
\]
\end{proof}

\section{Auxiliary Lemmas}
\begin{lemma} \label{lemma:1}
Define
\[
    \I_k^i = \left\{s \in [m]: [\Psi_k x_i]_s > 0 \right \}, \quad \I_k^j = \left\{s \in [m]: [\Psi_k x_j]_s > 0 \right \}
\]
where $\Psi_k \in \R^{m\times m}$ is a Gaussian random matrix. Suppose $|x_i| = 1$, $|x_j| = 1$, and $|x_i^{\top}x_j| \leq \delta$, where $\delta \leq 1/3$. Then, with a probability $1 - \exp(-\Omega(m))$, we have
\[
    \left|\I_k^i \cap \I_k^j \right| \leq \frac{m}{3}
\]
\end{lemma}
\begin{proof}
Define $[z]_+$ that outputs $1$ when $z > 0$, and zero otherwise. We have
\begin{eqnarray*}
    \left|\I_k^j \cap \I_k^i\right| & = & \langle\left[\Psi_kx_j\right]_+, \left[\Psi_kx_i\right]_+ \rangle \\
    & \leq & \langle\left[\Psi_kx_j\right]_+, \left[\Psi_k\left(x_i - P_{x_j}(x_i)\right)\right]_+ \rangle + \left|\left[\Psi_kx_i\right]_+ - \left[\Psi_k\left(x_i - P_{x_j}(x_i)\right)\right]_+\right|
\end{eqnarray*}
Since $\Psi_k(x_i - P_{x_j}(x_i))$ and $\Psi_k x_j$ are statistically independent, hence, with a probability $1 - \exp(-\theta m/2)$,
\[
\langle \left[\Psi x_j\right]_+, \left[\Psi_k(x_i - P_{x_j}(x_i)\right]_+\rangle \leq \frac{(1+\theta)m}{4}
\]
and therefore,
\begin{eqnarray*}
\lefteqn{\left|\I_k^j \cap \I_k^i\right|} \\
& \leq & \frac{(1+\theta)m}{4} + \left|\left[\Psi_kx_i\right]_+ - \left[\Psi_k\left(x_i - P_{x_j}(x_i)\right)\right]_+\right| \\
& \leq & \frac{(1+\theta)m}{4} + \underbrace{\sum_{\ell=1}^m I\left(\left|\left[\Psi_kP_{x_j}(x_i)\right]_{\ell}\right| > \left|\left[\Psi_k\left(x_i - P_{x_j}(x_i)\right)\right]_{\ell}\right|\right)}_{:= Z}
\end{eqnarray*}
We proceed to bound $Z$. First, each element of $\Psi_k P_{x_j}(x_i)$ follows a Gaussian distribution $\N(0, \sigma_1^2)$, where $\sigma_1^2 = 2|\Psi_k P_{x_j}(x_i)|^2/m$. Similarly, each element $\Psi_k (x_i - P_{x_j}(x_i))$ follows a distribution $\N(0, \sigma_2^2)$, where $\sigma_2^2 = 2|\Psi_k(x_i - P_{x_j}(x_i))|^2/m$. Using the concentration inequality of $\chi^2_m$ distribution, we have, with a probability $1 - 2\exp(-\theta m/4)$,
\[
    \sigma_1^2 \leq \frac{2(1+\theta)}{m}\left|P_{x_j}(x_i)\right|^2\leq\frac{2(1+\theta)\delta^2}{m}, \quad \sigma_2^2 \geq \frac{2(1-\theta)}{m}\left|\left(x_i - P_{x_j}(x_i)\right)\right|^2 \geq \frac{2(1-\theta)(1 - \delta^2)}{m}
\]
Using the properties of Gaussian distribution, we have
\begin{eqnarray*}
    \lefteqn{\Pr\left(\left|\left[\Psi_k\left(x_i - P_{x_j}(x_i)\right)\right]_{\ell}\right|^2 \geq \frac{2\delta}{m}\sqrt{\frac{(1 - \theta)(1-\delta^2)}{1+\theta}}\right)} \\
    & \geq & 1 - \left(\frac{(1+\theta)\delta^2}{(1-\theta)1 - \delta^2}\right)^{1/4} \exp\left(-\frac{1 - \left[(1+\theta)\delta^2/[(1 - \theta)(1-\delta^2)]\right]^{1/4}}{2}\right) \geq 1 - 2\delta^{1/4}
\end{eqnarray*}
\begin{eqnarray*}
    \lefteqn{\Pr\left(\left|\left[\Psi_k P_{x_j}(x_i)\right]_{\ell}\right|^2 \geq \frac{2\delta}{m}\sqrt{\frac{(1 - \theta)(1 - \delta^2)}{1+\theta}}{m}\right)} \\
     & \leq & \left(\frac{(1-\theta)(1 - \delta^2)}{(1 + \theta)\delta^2}\right)^{1/4} \exp\left(- \frac{\left[(1 - \theta)(1 - \delta^2)/(1+\theta)\delta^2\right]^{1/4}}{2}\right) \leq \exp\left(-\frac{1}{2\delta}\right)
\end{eqnarray*}
As a result, we have
\[
\Pr\left(\left|\left[\Psi_k\left(x_i - P_{x_j}(x_i)\right)\right]_{\ell}\right| < \left|\left[\Psi_k P_{x_j}(x_i)\right]_{\ell}\right|\right) \leq 3\delta^{1/4}\exp\left(-\frac{1}{2\delta}\right)
\]
Using the standard concentration inequality, we have, with a probability $1 - \exp(-m/2)$,
\[
    Z \leq \left(2\delta^{1/4}\exp\left(-\frac{1}{2\delta}\right) + \frac{1}{3.5}\right)m + 2m\sqrt{\delta^{1/4}\exp\left(-\frac{1}{2\delta}\right)} \leq \frac{m}{7}
\]
and therefore, with a probability $1 - \exp(-\Omega(m))$, we have
\[
\left|\I_k^j \cap \I_k^i\right| \leq \frac{m}{3}
\]
\end{proof}
\begin{lemma}\label{lemma:z-ka-kb-2}
With a probability at least {$ 1- 2L\exp(-\theta^2 m/[16L^2])$}, for any $1 \leq k_b \leq k_a \leq L$ and any $\theta \in (0, 1/2)$, we have
\[
    \nm{Z_{k_a,k_b}^{1,i}} _2\leq \sqrt{12 L}  e^{\theta/2}\theta^{-1/2}.
\]
\end{lemma}
\begin{proof}
To bound $\|Z_{k_a, k_b}^{1,i}\|_2$, we will bound $\max\limits_{\nm{v} \le 1} \nm{Z_{k_a, k_b}^{1,i}v}$. To this end, we divide $v$ into $r$ components of equal size, denoted by $v_1, \ldots, v_{r}$. We furthermore introduce $\tilde{v}_k \in \R^m$ to include the $k$th component $v_k$ of $v$ with everything else padded with zeros. As a result,
\begin{align*}
    \max\limits_{\nm{v} \leq 1} \nm{Z_{k_a,k_b}^{1,i}v} \leq& \max\limits_{\sum_{j=1}^{r} \nm{v_j}^2 \leq 1} \sum_{j=1}^{r} \nm{Z_{k_a,k_b}^{1,i}\tilde{v}_j} \\ 
    \leq& \max\limits_{\sum_{j=1}^{r}\nm{v_j}^2 \leq 1} \sum_{j=1}^{r}\left\{\nm{v_j}\max_{v_j \in \R^{m/r}} \frac{\nm{Z_{k_a,k_b}^{1,i} \tilde{v}_j}}{\nm{v_j}} \right\}\\
   \overset{(a)}{\le} & \sqrt{r}\max\limits_{\substack{\sum_{j=1}^{r}\nm{v_j}^2 \leq 1,\\ j \in [r], v_j \in \R^{m/r}}} \frac{\nm{Z_{k_a,k_b}^{1,i} \tilde{v}_j}}{\nm{v_j}}\\
    \overset{(b)}{\le} & \sqrt{r}\max\limits_{j \in [r], v_j \in \R^{m/r}} \frac{\nm{Z_{k_a,k_b}^{1,i} \tilde{v}_j}}{\nm{v_j}},
\end{align*}
where (a) uses the fact that $\max_{v_j \in \R^{m/r}} \frac{\nm{Z_{k_a,k_b}^{1,i} \tilde{v}_j}}{\nm{v_j}} \le \max_{j \in [r], v_j \in \R^{m/r}} \frac{\nm{Z_{k_a,k_b}^{1,i} \tilde{v}_j}}{\nm{v_j}}$ for any $j\in[r]$ and Cauchy-Schwarz inequality that $\sum_{j=1}^{r}\|v_j\|\le \sqrt{\sum_{j=1}^{r}1^2 \sum_{j=1}^{r}\|v_j\|^2}\le \sqrt{r}$ since $\sum_{j=1}^{r}\|v_j\|^2 = \|v\|^2 \le 1$; (b) holds since one constraint $\sum_{j=1}^{r}\nm{v_j}^2 \leq 1$ for maximization is removed.
We fix $j$ and consider a fixed $\tilde{v}_j$. We will prove the following high probability bound by induction
\begin{align}
    \nm{Z_{k_a, k_b}^{1,i} \tilde{v}_j}^2 \in \left(\left(1 - \frac{\theta}{L}\right)^{k_a - k_b}\nm{D_{k_b}^i\tilde{v}_j}^2, \left(1 + \frac{\theta}{L}\right)^{k_b-k_a}\nm{D_{k_b}^i\tilde{v}_j}^2\right),
\end{align}
where $\theta \in (0, 1/2)$ is a introduced parameter related to the failure probability. For $k_a = k_b +1$, we have
\[
    \nm{Z_{k_a, k_b}^{1,i}\tilde{v}_j}^2 = \nm{D^i_{k_b+1} W_{1,k_b+1} D^i_{k_b} \tilde{v}_j}^2
\]
It is easy to show that $\frac{m\nm{Z_{k_a, k_b}^{1,i}\tilde{v}_j}^2}{2\nm{D^i_{k_b}\tilde{v}_j}^2}$ follows the distribution of $\chi^2$ with degree of $m/2$. Using the concentration of $\chi^2$ distribution, we have
\begin{align*}
    \Pr\left(\nm{Z_{k_a,k_b}^{1,i}\tilde{v}_j}^2 \geq \left(1 + \frac{\theta}{L}\right)\nm{D^i_{k_b}\tilde{v}_j}^2\right) \leq \exp\left({-\frac{\theta^2 m}{8L^2}}\right),\\
    \Pr\left(\nm{Z_{k_a,k_b}^{1,i}\tilde{v}_j}^2 \leq \left(1 - \frac{\theta}{L}\right)\nm{D^i_{k_b}\tilde{v}_j}^2\right) \leq \exp\left({-\frac{\theta^2 m}{8L^2}}\right).   
\end{align*}
Hence, with a probability $1 - 2 \exp\left({-\theta^2 m/[8L^2]}\right)$, when $k_a =k_b + 1$, we have
\begin{align*}
    \nm{Z_{k_a, k_b}^{1,i}\tilde{v}_j}^2/\nm{D_{k_b}^i \tilde{v}_j}^2 =  \nm{Z_{k_b+1, k_b}^{1,i}\tilde{v}_j}^2/\nm{D_{k_b}^i \tilde{v}_j}^2 \in\left[1- \frac{\theta}{L}, 1 + \frac{\theta}{L}\right]. 
\end{align*}
As a result, for a fixed $\tilde{v}_j$, for any $k_a$ and $k_b$, with a probability $1 - {2 L\exp\left(-\theta^2 m/[8L^2]\right)}$, we have
\begin{align}
   \nm{Z_{k_a, k_b}^{1,i} \tilde{v}_j}^2 \in \left[\left(1 - \frac{\theta}{L}\right)^L\nm{D^i_{k_b}\tilde{v}_j}^2, \left(1 + \frac{\theta}{L}\right)^L\nm{D_{k_b}^i \tilde{v}_j}^2\right].
\end{align}
Since $\theta < 1/2$, the above bound can be simplified as
\begin{align}
    \nm{Z_{k_a, k_b}^{1,i} \tilde{v}_j}^2 \in \left[e^{-2\theta}\nm{D^i_{k_b}\tilde{v}_j}^2, e^{\theta}\nm{D_{k_b}^i \tilde{v}_j}^2\right],
\end{align}
where we uses two inequalities: $(1+\theta/L)^L \le \exp(\theta)$ for $L>1, |\theta|\le n$ and $\exp(-\theta)\le 1-\theta/2$ for $0<\theta<1$.

To bound $\nm{Z_{k_a, k_b}\tilde{v}_j}^2$ for any $\tilde{v}_j$, we introduce $\mathcal N \subseteq \R^{m/r - 1}$ as an appropriate $\epsilon$-net of $\R^{m/r - 1}$. Using the well known property of appropriate $\epsilon$-net, we have
\begin{align}
     \max\limits_{v_j \in \mathcal N} \nm{Z_{k_a, k_b}^{1,i}\tilde{v}_j}^2 \overset{(a)}{\leq} \max\limits_{v_j \in \mathcal N} \frac{\nm{Z_{k_a, k_b}^{1,i}\tilde{v}_j}^2}{\nm{v_j}^2} \overset{(b)}{\leq}\max\limits_{v_j \in \R^{m/r - 1}} \frac{\nm{Z_{k_a, k_b}^{1,i}\tilde{v}_j}^2}{\nm{v_j}^2} \overset{(c)}{\leq} \frac{1}{1 - 2\epsilon}\max\limits_{v_j \in \mathcal N} \nm{Z_{k_a, k_b}^{1,i}\tilde{v}_j}^2,
\end{align}
where (a) is due to $\nm{v_j}^2 \le 1$; (b) is due to $\mathcal N \subseteq \R^{m/r-1}$; (c) is due to the {the property of appropriate $\theta$-net.} 
{Since $|\mathcal N| \leq (3/\epsilon)^{m/r-1} < (1 + 3/\epsilon)^{m/r} \leq \exp(3m/[r\epsilon])$~\citep{pisier1999volume}}, by taking the union bound, we have, with probability at least {$ 1- 2L\exp(-\theta^2 m/[8L^2] + 3m/[r\epsilon])$},
\begin{align*}
\max\limits_{v_j \in \R^{m/r - 1}} \frac{\nm{Z_{k_a, k_b}^{1,i}\tilde{v}_j}^2}{\nm{v_j}^2} \in \left[e^{-2\theta}\frac{\nm{D^i_{k_b}\tilde{v}_j}^2}{\nm{v_j}^2}, \frac{e^{\theta}}{1 - 2\epsilon}\frac{\nm{D^i_{k_b}\tilde{v}_j}^2}{\nm{v_j}^2} \right].
\end{align*}
By choosing $r= \frac{48 L^2}{\theta^2\epsilon}$ with $\theta=\frac{1}{3}$, we have, with a probability at least {$ 1- 2L\exp(-\theta^2 m/[16L^2])$}
\begin{align}
\max\limits_{v_j \in \R^{m/r - 1}} \frac{\nm{Z_{k_a, k_b}^{1,i}\tilde{v}_j}^2}{\nm{v_j}^2} \in \left[e^{-2\theta}\frac{\nm{D^i_{k_b}\tilde{v}_j}^2}{\nm{v_j}^2}, 3e^{\theta}\frac{\nm{D^i_{k_b}\tilde{v}_j}^2}{\nm{v_j}^2} \right].
\end{align}
and therefore, with a probability {$ 1- 2L\exp(-\theta^2 m/[16L^2])$}, for any $1\leq k_b < k_a \leq L$, we have
\begin{align}
\nm{Z_{k_a, k_b}^{1,i}}_2 \leq \frac{\sqrt{12 L}  e^{\theta/2}}{\sqrt{\theta}}.
\end{align}
\end{proof}

\begin{lemma}\label{tech:lemm:1}
For any sequences $\{a_i, i= 1, 2, \dots, n\}$ and $\{b_{i,j}, i, j= 1, 2, \dots, n\}$, we have
    \begin{align}
    \sum_{i=1}^n \sum_{j \neq i}a_j b_{i,j} =  \sum_{i=1}^n a_i \sum_{j\neq i} b_{j,i},\label{tech:lemm:1:res:1}\\
    \sum_{i=1}^n \sum_{j \neq i}a_i b_{i,j} =  \sum_{i=1}^n a_i \sum_{j\neq i} b_{i,j}. \label{tech:lemm:1:res:2}
    \end{align}
where the notation $\sum_{j\ne i}$ means $\sum_{j\ne i} b_j = \sum_{j=1}^{n} b_j - b_i$.
\end{lemma}
\begin{proof}
First, we know
    \begin{align*}
        & \sum_{i=1}^n \sum_{j \neq i}a_j b_{i,j} \\
        = & \sum_{j \neq 1}a_j b_{1,j}  +  \sum_{j \neq 2}a_j b_{2,j} + \dots + \sum_{j \neq n}a_j b_{n,j}\\
        = & \left(\sum_{j=1}^n a_j b_{1,j} - a_1b_{1,1} \right)+  \left(\sum_{j=1}^n a_j b_{2,j} - a_2b_{2,2}\right) + \dots +\left( \sum_{j=1}^n a_j b_{n,j} - a_nb_{n,n}\right)\\
        = &\sum_{i=1}^n\sum_{j=1}^n a_j b_{i,j} -\sum_{i=1}^n a_i b_{i,i}
        =  \sum_{j=1}^n a_j \sum_{i=1}^n b_{i,j} -\sum_{i=1}^n a_i b_{i,i} \\
        = & \sum_{j=1}^n a_j \sum_{i\neq j} b_{i,j}
        =  \sum_{i=1}^n a_i \sum_{j\neq i} b_{j,i} \quad \text{(note that $b_{i,j} \neq b_{j,i}$)}
    \end{align*}
    We then proved (\ref{tech:lemm:1:res:1}). Next, we want to prove (\ref{tech:lemm:1:res:2}). To this end, we have
    \begin{align*}
        & \sum_{i=1}^n \sum_{j \neq i}a_i b_{i,j} \\
        = & \sum_{j \neq 1}a_1 b_{1,j}  +  \sum_{j \neq 2}a_2 b_{2,j} + \dots + \sum_{j \neq n}a_n b_{n,j}\\
        = & \left(\sum_{j=1}^n a_1 b_{1,j} - a_1b_{1,1} \right)+  \left(\sum_{j=1}^n a_2 b_{2,j} - a_2b_{2,2}\right) + \dots +\left( \sum_{j=1}^n a_n b_{n,j} - a_nb_{n,n}\right)\\
        = & \sum_{i=1}^n a_i\sum_{j=1}^n b_{i,j} -\sum_{i=1}^n a_i b_{i,i}
        =  \sum_{i=1}^n a_i \sum_{j\neq i} b_{i,j}.
    \end{align*}
\end{proof}

\begin{lemma} \label{lemma:z-ka-kb-t}
Assume
\[
    \sqrt{12}  e^{\theta/2}\theta^{-1/2}L^{3/2}\tau \leq \frac{1}{9}
\] 
where $\tau$ is defined in (\ref{def:tau}).
Then, for any $\theta \in (0, 1/2)$, with a probability $ 1- 4L^2\exp(-\theta^2 m/[16L^2])$, for any $1 \leq k_b < k_b \leq L$, we have
\[
    \|Z_{k_a, k_b}^{t,i}\|_2 \leq 4\sqrt{L}  e^{\theta/2}\theta^{-1/2}
\]
\end{lemma}
\begin{proof}
Expand $Z^t_{k_a, k_b}$ as
\begin{eqnarray*}
Z_{k_a,k_b}^{t,i} = Z_{k_a, k_b}^{1,i} + \sum_{s=1}^{k_b - k_a} \sum_{k_1 > k_2 >\ldots > k_s}Z_{k_a, k_1+1}^{1,i}\prod_{\ell=1}^s\delta W_{t,k_{\ell}}Z_{k_{\ell}+1, k_{\ell+1}}^{1,i}
\end{eqnarray*}
According to Lemma~\ref{lemma:z-ka-kb-2}, for any $\theta \in (0, 1/2)$, with a probability {$ 1- 2L\exp(-\theta^2 m/[16L^2])$}, we have $\|Z_{k_a,k_b}^{1,i}\| \leq \sqrt{12 L}  e^{\theta/2}\theta^{-1/2}$. We thus have, with a probability {$ 1- 4L^2\exp(-\theta^2 m/[16L^2])$},
\begin{align*}
    \|Z_{k_a,k_b}^{t,i}\|_2 \overset{(a)}{\le} &\|Z_{k_a, k_b}^{1,i}\|_2 + \sum_{s=1}^{k_b - k_a} \sum_{k_1 > k_2 >\ldots > k_s}\|Z_{k_a, k_1+1}^{1,i}\|_2\prod_{\ell=1}^s\|\delta W_{t,k_{\ell}}\|_2\|Z_{k_{\ell}+1, k_{\ell+1}}^{1,i}\|_2\\
    \overset{(b)}{\le} & \sqrt{12 L}  e^{\theta/2}\theta^{-1/2}\left(1 + \sum_{s=1}^{k_b - k_a} (_s^L)(\tau \sqrt{12 L}  e^{\theta/2}\theta^{-1/2})^s\right)\\
    \overset{(c)}{\le} & \sqrt{12 L}  e^{\theta/2}\theta^{-1/2}\left(1 + \sum_{s=1}^{k_b - k_a} (L\tau \sqrt{12 L}  e^{\theta/2}\theta^{-1/2})^s\right)\\
    \overset{(d)}{\le} & 4\sqrt{L}  e^{\theta/2}\theta^{-1/2},
\end{align*}
where (a) uses $\|AB\|_2\le\|A\|_2\|B\|_2$ and $\|A+B\|_2\le\|A\|_2+\|B\|_2$ for any matrices $A,B\in\R^{m\times m}$; (b) uses Lemma~\ref{lemma:z-ka-kb-2} upper to $2^L$ times that for all combinations of $\|Z_{k_a,k_b}^{1,i}\|_2 \le \sqrt{12 L}  e^{\theta/2}\theta^{-1/2}$ and the definition of $\tau$; (c) uses the facts that $(_s^L) \leq L^s$; (d) uses $\sum_{i=0}^{n}a^i \le\frac{1}{1-a}$ for $a\in(0,1)$ 
with $\sqrt{12 L}  e^{\theta/2}\theta^{-1/2}L\tau \leq \frac{1}{9}$, and $\frac{9\sqrt{3}}{4} < 4$.
\end{proof}
\end{document}